\documentclass[11pt]{article}

\usepackage[T1]{fontenc}
\usepackage[utf8]{inputenc}
\usepackage[margin=1in]{geometry}
\usepackage[final,expansion=false]{microtype}

\usepackage{amsmath,amssymb,amsthm,mathtools,bm}
\usepackage{dsfont}        
\usepackage{enumitem}
\usepackage{booktabs}
\usepackage{tabularx}
\usepackage{longtable}
\usepackage{array}
\usepackage{graphicx}
\usepackage{float}
\graphicspath{{figures/}}

\usepackage[round,authoryear]{natbib}
\usepackage{xcolor} 
\usepackage{hyperref}
\hypersetup{
	colorlinks=true,
	linkcolor=blue!60!black,
	citecolor=blue!60!black,
	urlcolor=blue!60!black
}

\numberwithin{equation}{section}

\newtheorem{theorem}{Theorem}[section]
\newtheorem{proposition}[theorem]{Proposition}
\newtheorem{lemma}[theorem]{Lemma}
\newtheorem{corollary}[theorem]{Corollary}

\theoremstyle{definition}
\newtheorem{definition}[theorem]{Definition}
\newtheorem{assumption}[theorem]{Assumption}
\newtheorem{example}[theorem]{Example}

\theoremstyle{remark}
\newtheorem{remark}[theorem]{Remark}

\newcommand{\RR}{\mathbb{R}}

\newcommand{\PP}{\mathbb{P}}
\newcommand{\EE}{\mathbb{E}}

\DeclareMathOperator{\Cov}{Cov}

\newcommand{\1}{\mathds{1}}

\newcommand{\cP}{\mathcal{P}}

\newcommand{\cS}{\mathcal{S}}

\newcommand{\cX}{\mathcal{X}}

\newcommand{\ip}[2]{\left\langle #1,\,#2 \right\rangle}
\newcommand{\norm}[1]{\left\lVert #1 \right\rVert}
\newcommand{\abs}[1]{\left\lvert #1 \right\rvert}

\date{}
\begin{document}
	
	\title{Soft-to-Hard Routing in Sparse Mixture-of-Experts Models}
	
	\author{Reza Rastegar\\
Meta Platforms, Inc\thanks{reza.j.rastegar@gmail.com; Woodinville, WA}}
	\maketitle
	
	\begin{abstract}
		Softmax routing approaches hard top-1 routing as the temperature tends to zero, but the limiting passage is singular at router ties. This paper develops a boundary-layer calculus for this soft-to-hard limit in population squared-loss mixture-of-experts regression. For a router with logits $a_k(x;\phi)$, the relevant local quantity is the top-two margin $\Delta(x;\phi)$, and the relevant global quantity is the boundary mass $\PP(\Delta(X;\phi)\le w)$. Under smoothness and transversality assumptions, coarea and tubular-neighborhood estimates show how this mass scales with the slab width; in the binary case the leading coefficient is an explicit surface integral over the routing interface. These geometric estimates give quantitative bounds between the soft objective $L_\tau$ and the hard objective $L_0$, including an $O(\tau^\alpha)$ uniform comparison under a margin-tail condition, and yield $\Gamma$-convergence of the soft objectives on compact parameter spaces. The main conclusion is that the zero-temperature approximation is controlled by the probability carried by an $O(\tau)$ neighborhood of the routing interfaces, not by temperature alone. After isolating this boundary-layer part of the problem, we record a conditional landscape-transfer theorem from hard to small-temperature soft routing and a reduced two-expert Gaussian calculation illustrating local symmetry breaking. Synthetic diagnostics are included only as controlled checks of the boundary-layer predictions.
	\end{abstract}
	

	
\section{Introduction}
	
	This paper studies a specific soft-to-hard limit in mixture-of-experts routing. As a softmax gate is cooled toward a hard top-1 gate, convergence of the population objective is controlled by the amount of input probability near router indifference. Small temperature alone is not the governing quantity. The soft and hard gates differ only where the router is close to being indifferent between two experts. The size of that region, measured under the input distribution, determines the approximation error. We call this quantity boundary mass, and the main purpose of the paper is to develop a calculus for it.
	
	Mixture-of-experts (MoE) models combine local predictors through a learned gate. In the classical statistical formulation, the gate learns which expert should explain which part of the input space \citep{JacobsJordanNowlanHinton1991,JordanJacobs1994,YukselWilsonGader2012}. In modern sparse MoE architectures, the same routing idea is used for conditional computation, where only a small number of experts are active for a given token or example and model capacity can increase faster than per-example cost \citep{Shazeer2017SGMoE,LepikhinEtAl2021GShard,FedusZophShazeer2021Switch,DuEtAl2022GLaM,ZhouEtAl2022ExpertChoice}. These two uses of MoE differ in scale and purpose, but they share a structural feature that is easy to lose sight of. The router does not merely weight already-fixed predictions. It creates decision interfaces in input space, and probability mass near those interfaces is precisely where soft and hard routing can disagree.
	
	To make the issue concrete, let $f_k(\cdot;\theta_k)$ be the experts and let $a_k(x;\phi)$ be the router logits, $k=1,\ldots,K$. At temperature $\tau>0$, the dense softmax gate assigns
	\begin{equation}\label{eq:intro-softmax-preview}
		p_k^{(\tau)}(x;\phi)=\frac{\exp(a_k(x;\phi)/\tau)}{\sum_{\ell=1}^K \exp(a_\ell(x;\phi)/\tau)},
		\qquad k=1,\dots,K,
	\end{equation}
	and the soft predictor is
	\begin{equation}\label{eq:intro-soft-predictor-preview}
		h_{\theta,\phi,\tau}(x)=\sum_{k=1}^K p_k^{(\tau)}(x;\phi)f_k(x;\theta_k).
	\end{equation}
	As $\tau\downarrow0$, the weights in~\eqref{eq:intro-softmax-preview} converge pointwise to a hard top-1 assignment at every input where the largest logit is unique. That pointwise convergence is elementary. The objective-level limit is subtler. If the top logit exceeds the second by much more than $\tau$, the softmax weights are already exponentially close to a hard assignment. If the gap is of order $\tau$, two experts may still receive comparable weight. Thus the relevant set is a thin boundary layer around the hard-routing interfaces.
	
	For a router parameter $\phi$, write
	\[
		\Delta(x;\phi)=\max_k a_k(x;\phi)-\max_{\ell\ne k_*(x;\phi)} a_\ell(x;\phi),
	\]
	where $k_*(x;\phi)$ is a maximizer, away from ties. The boundary layer at width $w$ is the event $\{\Delta(X;\phi)\le w\}$. The basic program of the paper is to estimate the probability of this event, convert that estimate into a soft-to-hard risk comparison, and use the comparison to pass from the soft variational problem to the hard one. In this sense the paper is closer to a boundary-layer analysis than to a general theory of MoE training. The landscape and symmetry-breaking results later in the paper are included only after the boundary-layer estimates have been separated from the additional model assumptions they require.
	
	\subsection{Contributions}
	
	The first contribution is geometric. For regular routing interfaces, Proposition~\ref{prop:coarea} gives a first-order coarea expansion for the probability of a score slab and identifies the leading coefficient as a surface integral over the interface. Lemma~\ref{lem:tube-single} records the corresponding tube estimate for one interface, while Proposition~\ref{prop:margin-tail-upper} turns pairwise tube bounds into a top-two margin-tail bound for a multi-expert router. Example~\ref{ex:gaussian-linear-router} gives the resulting constant explicitly for a linear router under Gaussian covariates. These statements are the core of the boundary-layer calculus.
	
	The second contribution is a quantitative soft-to-hard comparison. For a fixed regular binary router with bounded response and bounded experts, Corollary~\ref{cor:coarea-soft-hard-linear} gives an $O(\tau)$ risk bound. Theorem~\ref{thm:soft-hard-quant} proves the uniform version. If the router class satisfies a margin-tail condition
	\[
		\PP\bigl(\Delta(X;\phi)\le t\bigr)\le C t^\alpha
	\]
	uniformly in $\phi$, then $\sup_{\theta,\phi}|L_\tau(\theta,\phi)-L_0(\theta,\phi)|$ is $O(\tau^\alpha)$ under the stated boundedness assumptions. Theorem~\ref{thm:gamma} then uses this uniform approximation to obtain $\Gamma$-convergence of $L_\tau$ to the hard-routing objective $L_0$ on compact parameter spaces. This is the variational form of the soft-to-hard limit.
	
	The third contribution is to show how the same calculus can be used without overstating what it proves. Theorem~\ref{thm:benign} is a conditional landscape-transfer result. It assumes that the profiled hard-routing objective already has favorable geometry--realizability, separation modulo relabeling, curvature near the teacher solution, and negative curvature away from teacher-equivalent critical points--and that first- and second-order soft-to-hard transfer holds in the relevant region. Under these hypotheses, the small-temperature soft objective inherits the corresponding recovery and strict-saddle structure. The theorem is therefore a stability statement around a well-behaved hard problem, not a claim that arbitrary MoE objectives have benign landscapes.
	
	The fourth contribution is explanatory rather than foundational. Proposition~\ref{prop:symmetry-breaking} studies a reduced two-expert Gaussian population model near a balanced router and computes the local router update when the experts already have a small contrast. The calculation displays an unstable direction aligned with the teacher separator. It is not used to prove the boundary-layer or landscape theorems; its role is to make visible, in a tractable setting, how expert contrast can select a routing interface. Section~\ref{sec:experiments} reports synthetic diagnostics that check the same boundary-layer predictions in controlled examples. These experiments are not presented as sparse-MoE benchmarks.
	
	\subsection{Related literature}
	
	MoE models were introduced as adaptive combinations of local experts with learned gates \citep{JacobsJordanNowlanHinton1991}, and hierarchical mixtures connected the architecture to latent allocation and EM-style training \citep{JordanJacobs1994}. The survey of \citet{YukselWilsonGader2012} gives a broad account of earlier models and applications. Modern sparse MoE layers use routing as conditional computation inside deep networks; sparsely gated MoE layers, GShard, Switch Transformers, GLaM, and expert-choice routing are representative examples \citep{Shazeer2017SGMoE,LepikhinEtAl2021GShard,FedusZophShazeer2021Switch,DuEtAl2022GLaM,ZhouEtAl2022ExpertChoice}.
	
	A complementary line of work develops statistical theory for MoE estimation. Recent papers study identifiability, least-squares estimation, softmax and sigmoid gates, sparse top-$K$ gates, cosine routing, and weakly identifiable or over-specified regimes \citep{HoYangJordan2022,NguyenNguyenHo2023,NguyenAkbarianYanHo2024,NguyenAkbarianNguyenHo2024,NguyenHoRinaldo2024LSE,NguyenHoRinaldo2024Sigmoid,NguyenEtAl2025Cosine}. Other works analyze routing as a mechanism for optimization and specialization, including stylized settings where gates help discover latent clusters or stabilize expert specialization \citep{Makkuva2019Gridlock,Makkuva2020Gated,ChenDengWuGuLi2022,DikkalaEtAl2023,KawataEtAl2025}. Conditional computation and expressivity give another perspective \citep{Shazeer2017SGMoE,FedusZophShazeer2021Switch,DuEtAl2022GLaM,WangE2025ExpressiveMoE}.
	
	The present paper asks a different question. It does not analyze finite-sample estimation rates, propose a training algorithm, or compare large-scale architectures. It isolates the low-temperature passage in the population objective. The mathematical tools therefore come from coarea formulas, tube estimates, variational convergence, and stability of critical-point structure \citep{EvansGariepy2015,HenrotPierre2005,SokolowskiZolesio1992,Braides2002,Ge2015,Sun2017Benign}. The resulting message is narrow but useful. Once the hard partition is fixed or uniformly controlled, the discrepancy between soft and hard routing is governed by boundary mass near the router interfaces.
	
	\subsection{Organization}
	
	Section~\ref{sec:setup} fixes the population model, the soft and hard risks, and the margin notation. Section~\ref{sec:global-geometry} proves the coarea and tube estimates that quantify boundary mass. Section~\ref{sec:gamma} turns those estimates into soft-to-hard risk bounds and $\Gamma$-convergence. Section~\ref{sec:benign} gives the conditional landscape-transfer theorem. Section~\ref{sec:symmetry-breaking} contains the reduced Gaussian symmetry-breaking calculation. Section~\ref{sec:experiments} reports controlled diagnostic experiments, and Section~\ref{sec:discussion} summarizes the scope of the results and the assumptions needed beyond the boundary-layer calculus.

	\section{Population MoE model and notation}\label{sec:setup}
	
	The analysis is carried out at the population level, with scalar response and squared loss. This choice is intentional. The question in the paper is not how an empirical optimizer behaves on a finite sample, but how the routing objective changes as the softmax temperature is lowered. The input law is assumed regular only where the coarea and tube estimates need such regularity. Empirical risk, generalization error, regularization, and load-balancing penalties are all important in trained MoE systems, but including them here would mix several effects with the specific zero-temperature issue studied in the paper.
	
	Let $\cX\subseteq\RR^d$ be the input space, equipped with the Euclidean norm $\|\cdot\|$, and let $(X,Y)$ have joint law $\mathsf{P}$ on $\cX\times\RR$. Its $X$-marginal is denoted by $P_X$. For events $A\subseteq\cX$ depending only on $X$, we write interchangeably $P_X(A)$ and $\PP(A)$. For a scalar predictor $f:\cX\to\RR$, the population squared-loss risk is
	\begin{equation}\label{eq:pop-risk}
		\mathcal{R}(f) := \EE_{(X,Y)\sim \mathsf{P}}\bigl[(Y-f(X))^2\bigr].
	\end{equation}
	Unless explicitly stated otherwise, expectations are taken with respect to the joint law $\mathsf{P}$, while probabilities involving only $X$ are written either as $P_X(\cdot)$ or, equivalently, as $\PP(\cdot)$.
	
	Fix $K\ge 2$ experts, indexed by $k\in\{1,\dots,K\}$. Let $\Theta_k$ denote the parameter space of expert~$k$, let $\Theta:=\prod_{k=1}^K\Theta_k$, and let $\Phi$ denote the parameter space of the router. An MoE predictor consists of scalar expert maps $f_k(\cdot;\theta_k):\cX\to\RR$ with $\theta=(\theta_1,\dots,\theta_K)\in\Theta$, together with a router parameter $\phi\in\Phi$ that assigns input-dependent weights to those experts through the logit map
	\begin{equation}\label{eq:logit-map}
	a(x;\phi)=\bigl(a_1(x;\phi),\dots,a_K(x;\phi)\bigr)\in\RR^K.
	\end{equation}
	For $\tau>0$, the router produces the temperature-$\tau$ softmax weights
	\begin{equation}\label{eq:softmax-router}
		p_k^{(\tau)}(x;\phi)\ :=\ \frac{\exp\!\bigl(a_k(x;\phi)/\tau\bigr)}{\sum_{j=1}^K \exp\!\bigl(a_j(x;\phi)/\tau\bigr)},
		\qquad k=1,\dots,K.
	\end{equation}
	These weights define the soft MoE predictor
	\begin{equation}\label{eq:moe-predictor}
		h_{\theta,\phi,\tau}(x)\ :=\ \sum_{k=1}^K p_k^{(\tau)}(x;\phi)\, f_k(x;\theta_k).
	\end{equation}
	The corresponding soft population risk is
	\begin{equation}\label{eq:soft-risk-def}
		L_\tau(\theta,\phi)
		:=\mathcal R(h_{\theta,\phi,\tau})
		=\EE\bigl[(Y-h_{\theta,\phi,\tau}(X))^2\bigr].
	\end{equation}
	All notation in this section is population-level. In particular, the definitions below should not be read as claims about empirical minimization, finite-sample generalization, or auxiliary penalties such as load balancing and entropy regularization.
	
	The basic geometric quantities already appear at this level. For a fixed router parameter $\phi$, let
	\[
	a_{(1)}(x;\phi)\ge a_{(2)}(x;\phi)\ge\cdots\ge a_{(K)}(x;\phi)
	\]
	denote the ordered logits. The top--two margin is
	\begin{equation}\label{eq:routing-margin}
		\Delta(x;\phi)\ :=\ a_{(1)}(x;\phi)-a_{(2)}(x;\phi)\ \ge\ 0.
	\end{equation}
	Large values of $\Delta(x;\phi)$ indicate that the winning expert is well separated from the runner-up, hence that routing is locally stable. The corresponding set of winning experts is
	\begin{equation}\label{eq:argmax-set}
	\mathcal{M}(x;\phi)\ :=\ \arg\max_{k\in\{1,\dots,K\}} a_k(x;\phi),
	\end{equation}
	and the routing boundary, or tie set, is
	\begin{equation}\label{eq:boundary-set}
		\mathcal{T}(\phi)\ :=\ \{x\in\cX:\ |\mathcal{M}(x;\phi)|\ge 2\}
		\ =\ \{x\in\cX:\ \Delta(x;\phi)=0\}.
	\end{equation}
	We use the deterministic hard winner
	\begin{equation}\label{eq:hard-winner}
		k^\star(x;\phi):=\min \mathcal M(x;\phi),
	\end{equation}
	where the minimum only fixes a tie-breaking convention on $\mathcal T(\phi)$. The associated hard predictor and hard population risk are
	\begin{align}
		h_{\theta,\phi,0}(x)
		&:= f_{k^\star(x;\phi)}(x;\theta_{k^\star(x;\phi)}),\label{eq:hard-predictor}\\
		L_0(\theta,\phi)
		&:=\EE\bigl[(Y-h_{\theta,\phi,0}(X))^2\bigr].\label{eq:hard-risk-def}
	\end{align}
	Away from $\mathcal{T}(\phi)$ the winning expert is unique, and hard routing is locally constant in $x$ under mild regularity of the logits.
	
	For later coarea estimates it is useful to track not only the top--two gap but also the smallest pairwise logit separation,
	\begin{equation}\label{eq:setup-pairwise-min-margin}
	\Delta_{\min}(x;\phi):=\min_{k\neq \ell}\abs{a_k(x;\phi)-a_\ell(x;\phi)}.
	\end{equation}
	Given a width parameter $w>0$, we introduce the top--two and pairwise boundary-mass profiles
	\begin{align}
	\mathrm{BM}^{\mathrm{top}}_\phi(w)&:=\PP\bigl(\Delta(X;\phi)\le w\bigr),\label{eq:BM-top-profile}\\
	\mathrm{BM}^{\mathrm{pair}}_\phi(w)&:=\PP\bigl(\Delta_{\min}(X;\phi)\le w\bigr).\label{eq:BM-pair-profile-setup}
	\end{align}
	These are related but not interchangeable. The top--two profile in~\eqref{eq:BM-top-profile} is the natural object for soft-versus-hard routing, because it measures how close the winning logit is to the runner-up. The pairwise profile in~\eqref{eq:BM-pair-profile-setup} is a geometric upper envelope: since $\Delta_{\min}(x;\phi)\le \Delta(x;\phi)$, it can also count near-ties between two nonwinning experts, but it decomposes cleanly into pairwise slabs around the surfaces $a_k=a_\ell$. We do not introduce a separate temperature-indexed pairwise boundary-mass symbol; when a tolerance $\varepsilon\in(0,1/2)$ and temperature $\tau$ are fixed, the pairwise slab probability is written explicitly as $\mathrm{BM}^{\mathrm{pair}}_\phi(\kappa_\varepsilon\tau)$, where
	\begin{equation}\label{eq:kappa-eps}
	\kappa_\varepsilon:=\log\!\Bigl(\frac{1-\varepsilon}{\varepsilon}\Bigr).
	\end{equation}
	Thus temperature enters the boundary-mass profiles only through the width argument $w=\kappa_\varepsilon\tau$.
	
	This viewpoint explains why the margin distribution controls the soft-to-hard limit. Fix $(x,\phi)$ and assume the maximizing logit is unique, so that $\Delta(x;\phi)>0$. A direct softmax estimate gives
	\begin{equation}\label{eq:setup-softmax-tail-preview}
	\sum_{j\neq k^\star(x;\phi)} p_j^{(\tau)}(x;\phi)
	\le (K-1)\exp\bigl(-\Delta(x;\phi)/\tau\bigr),
	\end{equation}
	with $k^\star$ as in~\eqref{eq:hard-winner}. In particular, $p^{(\tau)}(x;\phi)\to e_{k^\star(x;\phi)}$ as $\tau\downarrow0$. Thus the integrated discrepancy between soft and hard routing is governed by the amount of input mass for which $\Delta(X;\phi)$ is small, which is exactly the boundary-mass profile above.
	
	Two structural symmetries will matter later. First, the parameterization is invariant under permutations of the experts. If $\pi$ is a permutation of $\{1,\dots,K\}$ and one simultaneously relabels the expert parameters and the router outputs by $\pi$, then the resulting predictor $h_{\theta,\phi,\tau}$ is unchanged. Consequently, any isolated minimizer comes with at least $K!$ equivalent copies. Second, the logits are defined only up to a common translation. Replacing $a_k(x;\phi)$ by $a_k(x;\phi)+c(x)$ for an arbitrary scalar function $c(x)$ leaves the softmax weights unchanged, so one may fix this gauge when convenient, for example by enforcing $\sum_{k=1}^K a_k(x;\phi)=0$.
	
	We close the setup with two regularity assumptions. They are kept deliberately modest here. The first only ensures that the risks being written are well defined. The second is not needed for the softmax algebra itself; it enters only when routing boundaries are treated as geometric objects.
	
	\begin{assumption}[Measurability and integrability]\label{ass:meas-int}
		For each temperature $\tau>0$ and parameters $(\theta,\phi)$ under consideration,
		the map $x\mapsto h_{\theta,\phi,\tau}(x)$ is measurable and
		$\EE[(Y-h_{\theta,\phi,\tau}(X))^2]<\infty$.
	\end{assumption}
	
	Assumption~\ref{ass:meas-int} is largely bookkeeping. It ensures that the population risk is well defined for the parameter values under discussion.
	
	\begin{assumption}[Local regularity of logits]\label{ass:logit-regularity}
		When we discuss geometric properties of $\mathcal{T}(\phi)$ (e.g.\ codimension, tubular neighborhoods),
		we assume $a(\cdot;\phi)$ is locally Lipschitz (or $C^1$) on $\cX\subseteq\RR^d$.
	\end{assumption}
	
	Assumption~\ref{ass:logit-regularity} enters only when we study the geometry of routing boundaries. The later probabilistic soft-to-hard bounds are driven by the margin distribution and do not require this smoothness by themselves.

	
	\section{Boundary-layer geometry and coarea control}\label{sec:global-geometry}
	
	The first step is to understand the geometry of the router on its own, before introducing any teacher--student structure. Soft and hard routing differ in a meaningful way only near inputs where two scores are close. The same region controls how a hard partition reacts to a small movement of its boundary: away from the interface, the assigned expert is unchanged. This section makes that idea quantitative. We measure the probability mass in thin score slabs, use the coarea formula to identify the leading surface term, and then convert pairwise slab estimates into a margin-tail bound for multi-expert routers.
	
	\subsection{Boundary mass and coarea formula}\label{subsec:global-prelim}

	Fix a router parameter $\phi$ for the moment. We work on the input domain $\cX\subseteq\RR^d$ from Section~\ref{sec:setup}, write $P_X$ for the marginal law of $X$, and write $p_X$ for its density when such a density exists. The router scores are the logits $a_k(\cdot;\phi):\cX\to\RR$, $k=1,\dots,K$, and the temperature is denoted by $\tau>0$. Holding $\phi$ fixed in this subsection lets us separate a purely geometric calculation from later optimization questions.
	
	For each pair $(k,\ell)$ with $k\neq\ell$, define the \emph{pairwise score difference}
	\begin{equation}\label{eq:pairwise-score-difference}
		S_{k\ell}(x;\phi)
		:= a_k(x;\phi) - a_\ell(x;\phi),\qquad x\in\cX.
	\end{equation}
	When no ambiguity can arise, we abbreviate $S_{k\ell}(x;\phi)$ by $S_{k\ell}(x)$. The \emph{decision surface} between experts $k$ and $\ell$ is
	\begin{equation}\label{eq:pairwise-decision-surface}
		\cS_{k\ell}(\phi)
		:= \{x\in\cX : S_{k\ell}(x;\phi)=0\}.
	\end{equation}
	For the fixed router parameter used in this subsection, we also write $\cS_{k\ell}$ for $\cS_{k\ell}(\phi)$. The stronger tubular regularity needed for the coarea expansion is stated explicitly in Assumption~\ref{ass:router-coarea-regularity}; the present notation only identifies the pairwise interfaces to which that assumption will be applied.
	
	Let $p_X$ denote the density of $P_X$ with respect to Lebesgue measure on
	\(\RR^d\). We also use the pairwise minimum margin from~\eqref{eq:setup-pairwise-min-margin}, now written in terms of the score differences~\eqref{eq:pairwise-score-difference},
	\begin{equation}\label{eq:global-margin}
		\Delta_{\min}(x;\phi)
		:= \min_{k\neq\ell} |S_{k\ell}(x;\phi)|,
		\qquad x\in\cX.
	\end{equation}
	Small values of $\Delta_{\min}(X;\phi)$ identify inputs lying in thin slabs around at least one decision surface.
	
	\begin{definition}[Pairwise boundary mass profile]\label{def:boundary-mass-pair}
		For $w>0$ recall the \emph{pairwise boundary mass profile} from~\eqref{eq:BM-pair-profile-setup}, namely
		\begin{equation}\label{eq:BM-pair-profile}
			\mathrm{BM}^{\mathrm{pair}}_\phi(w)
			:= \PP\bigl(\Delta_{\min}(X;\phi)\le w\bigr).
		\end{equation}
		For a different input law $R$ on $\cX$, the same definition applies with $R(\cdot)$ in place of $P_X(\cdot)$.
	\end{definition}
	
	Fix $\varepsilon\in(0,1/2)$ and recall $\kappa_\varepsilon$ from~\eqref{eq:kappa-eps}. For each pair of experts define the \emph{pairwise ambiguity event}
	\begin{equation}\label{eq:pairwise-ambiguity-event}
		A_{k\ell}(\varepsilon,\tau)
		:= \Bigl\{
		\frac{\exp(a_k(X;\phi)/\tau)}
		{\exp(a_k(X;\phi)/\tau)+\exp(a_\ell(X;\phi)/\tau)}
		\in[\varepsilon,1-\varepsilon]
		\Bigr\},
		\qquad k\neq \ell.
	\end{equation}
	Then $A_{k\ell}(\varepsilon,\tau)$ is equivalent to $\abs{S_{k\ell}(X;\phi)}\le \kappa_\varepsilon\tau$, and the union of pairwise ambiguous events has probability
	\begin{equation}\label{eq:pairwise-ambiguity-slab-prob}
		\PP\Bigl(\bigcup_{1\le k<\ell\le K} A_{k\ell}(\varepsilon,\tau)\Bigr)
		=\PP\bigl(\Delta_{\min}(X;\phi)\le \kappa_\varepsilon\tau\bigr)
		=\mathrm{BM}^{\mathrm{pair}}_\phi(\kappa_\varepsilon\tau).
	\end{equation}
	Thus one does not need a new object for every temperature. The temperature only selects the slab width $\kappa_\varepsilon\tau$ at which the same boundary-mass profile is evaluated. This remains a pairwise slab quantity; it is not the same object as the $K$-way softmax ambiguity event introduced later in~\eqref{eq:Kway-ambiguity-region}.
	
	\begin{remark}[Equivalent slab form]\label{rem:BM-slab}
		Equation~\eqref{eq:pairwise-ambiguity-slab-prob} is the precise sense in which the pairwise boundary-mass profile measures an $O(\tau)$-thick union of slabs around the decision surfaces $\{\cS_{k\ell}\}$: the temperature and tolerance only determine the width $\kappa_\varepsilon\tau$ at which the profile is evaluated.
	\end{remark}
	
	For the binary case $K=2$ the definition simplifies considerably and gives a
	concrete geometric picture.
	
	\begin{remark}[Binary case]
		When $K=2$, we have
		\[
		\frac{e^{a_1(x;\phi)/\tau}}{e^{a_1(x;\phi)/\tau}+e^{a_2(x;\phi)/\tau}}
		= \sigma\Bigl(\frac{S_{12}(x)}{\tau}\Bigr),
		\]
		where $\sigma(t)=(1+e^{-t})^{-1}$ is the logistic function. Then
		\[
		\mathrm{BM}^{\mathrm{pair}}_\phi(\kappa_\varepsilon\tau)
		= \PP\bigl( S_{12}(X)
		\in[\underline{s}_\varepsilon(\tau),\overline{s}_\varepsilon(\tau)]\bigr),
		\]
		where $\underline{s}_\varepsilon(\tau)
		= \tau\log(\varepsilon/(1-\varepsilon))$ and
		$\overline{s}_\varepsilon(\tau)
		= \tau\log((1-\varepsilon)/\varepsilon)$. Thus the profile value $\mathrm{BM}^{\mathrm{pair}}_\phi(\kappa_\varepsilon\tau)$ is exactly the probability mass in an
		$O(\tau)$-thick slab around the decision surface~$\cS_{12}$.
	\end{remark}
	
	The coarea calculation needs the pairwise interfaces to be regular in a small neighborhood of the tie surface, not only exactly on the tie surface. We record that requirement next.
	
	\begin{assumption}[Regular router, density, and tubular level sets]\label{ass:router-coarea-regularity}
		For each $k\neq\ell$, the function $S_{k\ell}:\cX\to\RR$ is $C^2$ on a neighborhood of the tube $\{|S_{k\ell}|\le \delta_{\mathrm{reg}}\}$ for some $\delta_{\mathrm{reg}}>0$. There is a constant $c_0>0$ such that
		\begin{equation}\label{eq:grad-lower-tube}
			\norm{\nabla S_{k\ell}(x)}\ge c_0
			\qquad\text{for almost every $x\in\cX$ with }|S_{k\ell}(x)|\le \delta_{\mathrm{reg}},
		\end{equation}
		where ``almost every'' is with respect to Lebesgue measure. The density $p_X$ is continuous and bounded on $\cX$. In addition, the nearby level sets $S_{k\ell}^{-1}(t)$, $|t|\le\delta_{\mathrm{reg}}$, admit a regular tubular parametrization by normal flow away from a set of $\mathcal H^{d-1}$-measure zero, and the corresponding surface Jacobians are locally dominated by an integrable function. For the estimates below, the needed consequence is that the level-set profile
		\begin{equation}\label{eq:level-set-profile}
			F_{k\ell}(t):=
			\int_{\cX\cap S_{k\ell}^{-1}(t)}
			\frac{p_X(x)}{\norm{\nabla S_{k\ell}(x)}}
			\,\mathrm{d}\mathcal{H}^{d-1}(x)
		\end{equation}
		is finite for $|t|\le\delta_{\mathrm{reg}}$ and continuous at $t=0$.
	\end{assumption}

	Assumption~\ref{ass:router-coarea-regularity} combines two familiar requirements. First, tie surfaces should not become flat in a way that kills the transverse gradient. Second, the input law should not pile up on lower-dimensional sets near those surfaces. Together these conditions keep the interfaces nondegenerate in the only region that matters for the tube estimate. For smooth routers the lower bound \eqref{eq:grad-lower-tube} is the expected situation away from a negligible set of critical points, and the density condition covers continuous data laws as well as discrete data after a small amount of smoothing. Once these ingredients are in place, the usual tubular-neighborhood and coarea arguments give linear control of boundary-tube mass.
	
	One concrete way to diagnose failure of this assumption is to examine the empirical margin histogram together with the distribution of $\|\nabla_x S_{k\ell}(X;\phi)\|$ on samples for which $|S_{k\ell}(X;\phi)|$ is small. Persistent mass near zero combined with tiny gradients is the signature of a nearly degenerate interface.
	
	The point of Assumption~\ref{ass:router-coarea-regularity} is that it extends the nonvanishing-gradient condition from the interface itself to an entire tubular neighborhood $\{|S_{k\ell}|\le\delta_{\mathrm{reg}}\}$. That is the scale on which the argument actually runs, because one needs uniform control of $p_X/\|\nabla S_{k\ell}\|$ on nearby level sets rather than only on the zero set. The continuity of the profile in~\eqref{eq:level-set-profile} is automatic in many standard compact $C^2$ situations, and it is the precise regularity needed for the first-order tube expansion below.
	
	Under Assumption~\ref{ass:router-coarea-regularity}, each level set $S_{k\ell}^{-1}(t)$ is a smooth hypersurface for $|t|\le \delta_{\mathrm{reg}}$ up to a null set. The coarea formula then rewrites the mass of a slab by first integrating along these level sets and then integrating over the level value~$t$. In the binary case this gives an especially transparent expression: the leading coefficient is the surface integral of $p_X/\|\nabla S_{12}\|$ over the decision interface.
	
	\begin{proposition}[Coarea expansion for the binary pairwise boundary-mass profile]\label{prop:coarea}
		Suppose Assumption~\ref{ass:router-coarea-regularity} holds and $K=2$. Then for each
		$\varepsilon\in(0,1/2)$ we have
		\[
		\mathrm{BM}^{\mathrm{pair}}_\phi(\kappa_\varepsilon\tau)
		= C_\varepsilon\,\tau + o(\tau),
		\qquad\tau\downarrow0,
		\]
		where
		\[
		C_\varepsilon
		= 2\log\!\Bigl(\frac{1-\varepsilon}{\varepsilon}\Bigr)
		\int_{\cS_{12}}
		\frac{p_X(x)}{\norm{\nabla S_{12}(x)}}
		\,\mathrm{d}\mathcal{H}^{d-1}(x),
		\]
		and $\mathcal{H}^{d-1}$ denotes $(d-1)$-dimensional Hausdorff (surface)
		measure.
	\end{proposition}
	
	\begin{proof}
		Let $\underline{s}_\varepsilon(\tau)
		= \tau\log(\varepsilon/(1-\varepsilon))$ and
		$\overline{s}_\varepsilon(\tau)
		= \tau\log((1-\varepsilon)/\varepsilon)$. Then by the binary representation
		in the remark above,
		\[
		\mathrm{BM}^{\mathrm{pair}}_\phi(\kappa_\varepsilon\tau)
		= \PP\bigl(S_{12}(X)\in[\underline{s}_\varepsilon(\tau),\overline{s}_\varepsilon(\tau)]\bigr)
		= \int_{\cX}
		\1_{[\underline{s}_\varepsilon(\tau),\overline{s}_\varepsilon(\tau)]}(S_{12}(x))
		p_X(x)\,\mathrm{d}x.
		\]
		We now introduce the auxiliary function
		\[
		q_\tau(x)
		:= \1_{[\underline{s}_\varepsilon(\tau),\overline{s}_\varepsilon(\tau)]}(S_{12}(x))
		\frac{p_X(x)}{\norm{\nabla S_{12}(x)}},
		\]
		which is measurable and bounded for small~$\tau$ because $p_X$ is bounded and
		$\norm{\nabla S_{12}}$ is bounded away from zero on the relevant level sets by
		\eqref{eq:grad-lower-tube}. Then
		\[
		\mathrm{BM}^{\mathrm{pair}}_\phi(\kappa_\varepsilon\tau)
		= \int_{\cX} q_\tau(x)\,\norm{\nabla S_{12}(x)}\,\mathrm{d}x.
		\]
		By the coarea formula for $C^1$ (indeed Lipschitz) functions
		\citep{EvansGariepy2015},
		\[
		\int_{\cX} q_\tau(x)\,\norm{\nabla S_{12}(x)}\,\mathrm{d}x
		= \int_{\RR}
		\biggl[\int_{\cX\cap S_{12}^{-1}(t)} q_\tau(x)\,\mathrm{d}\mathcal{H}^{d-1}(x)\biggr]
		\mathrm{d}t.
		\]
		By the definition of $q_\tau$, the integrand vanishes unless
		$t\in[\underline{s}_\varepsilon(\tau),\overline{s}_\varepsilon(\tau)]$. Hence
		\[
		\mathrm{BM}^{\mathrm{pair}}_\phi(\kappa_\varepsilon\tau)
		= \int_{\underline{s}_\varepsilon(\tau)}^{\overline{s}_\varepsilon(\tau)}
		\biggl[\int_{\cX\cap S_{12}^{-1}(t)}
		\frac{p_X(x)}{\norm{\nabla S_{12}(x)}}\,
		\mathrm{d}\mathcal{H}^{d-1}(x)\biggr]
		\mathrm{d}t.
		\]
		Let $F(t):=F_{12}(t)$ be the level-set profile from~\eqref{eq:level-set-profile}. By Assumption~\ref{ass:router-coarea-regularity}, $F(t)$ is finite for $|t|\le\delta_{\mathrm{reg}}$ and continuous at $t=0$.
		For all sufficiently small $\tau$, the interval $[\underline{s}_\varepsilon(\tau),\overline{s}_\varepsilon(\tau)]$ is contained in $(-\delta_{\mathrm{reg}},\delta_{\mathrm{reg}})$. Hence
		\[
		\int_{\underline{s}_\varepsilon(\tau)}^{\overline{s}_\varepsilon(\tau)} F(t)\,\mathrm{d}t
		= \bigl(\overline{s}_\varepsilon(\tau)-\underline{s}_\varepsilon(\tau)\bigr)F(0)
		+ \int_{\underline{s}_\varepsilon(\tau)}^{\overline{s}_\varepsilon(\tau)}\bigl(F(t)-F(0)\bigr)\,\mathrm{d}t.
		\]
		The absolute value of the second term is at most
		\[
		\bigl(\overline{s}_\varepsilon(\tau)-\underline{s}_\varepsilon(\tau)\bigr)
		\sup_{|t|\le \kappa_\varepsilon\tau}|F(t)-F(0)|
		= o\bigl(\overline{s}_\varepsilon(\tau)-\underline{s}_\varepsilon(\tau)\bigr),
		\]
		because $F$ is continuous at zero. Therefore
		\[
		\mathrm{BM}^{\mathrm{pair}}_\phi(\kappa_\varepsilon\tau)
		= \bigl(\overline{s}_\varepsilon(\tau)-\underline{s}_\varepsilon(\tau)\bigr)\,F(0)
		+ o\bigl(\overline{s}_\varepsilon(\tau)-\underline{s}_\varepsilon(\tau)\bigr).
		\]
		Finally,
		\[
		\overline{s}_\varepsilon(\tau)-\underline{s}_\varepsilon(\tau)
		= 2\tau\log\Bigl(\frac{1-\varepsilon}{\varepsilon}\Bigr),
		\]
		which yields the desired linear expansion with
		\[
		C_\varepsilon
		= 2\log\Bigl(\frac{1-\varepsilon}{\varepsilon}\Bigr)
		F(0)
		= 2\log\Bigl(\frac{1-\varepsilon}{\varepsilon}\Bigr)
		\int_{\cS_{12}}
		\frac{p_X(x)}{\norm{\nabla S_{12}(x)}}\,
		\mathrm{d}\mathcal{H}^{d-1}(x).
		\]
	\end{proof}
	
	The constant $C_\varepsilon$ has a simple interpretation. The factor $2\log((1-\varepsilon)/\varepsilon)$ is the slab width in score units, while the surface integral measures how much input probability sits near the interface after correcting for the local score slope. The same argument gives the analogous expansion for any sufficiently thin symmetric slab.
	
	\begin{lemma}[General tube formula for a single decision surface]\label{lem:tube-single}
		Suppose Assumption~\ref{ass:router-coarea-regularity} holds for a given pair
		$(k,\ell)$. Then there exists $w_0>0$ such that for all
		$w\in(0,w_0]$,
		\[
		\PP\bigl(|S_{k\ell}(X)|\le w\bigr)
		= 2w
		\int_{\cS_{k\ell}}
		\frac{p_X(x)}{\norm{\nabla S_{k\ell}(x)}}
		\,\mathrm{d}\mathcal{H}^{d-1}(x)
		+ o(w),
		\qquad w\downarrow0.
		\]
	\end{lemma}
	
	\begin{proof}
		Apply the argument of Proposition~\ref{prop:coarea} with
		$[\underline{s}_\varepsilon(\tau),\overline{s}_\varepsilon(\tau)]$ replaced by $[-\delta,\delta]$ and
		$S_{12}$ replaced by $S_{k\ell}$. This gives
		\[
		\PP\bigl(|S_{k\ell}(X)|\le\delta\bigr)
		= \int_{-\delta}^{\delta} F_{k\ell}(t)\,\mathrm{d}t,
		\]
		where $F_{k\ell}$ is the profile defined in~\eqref{eq:level-set-profile}. Assumption~\ref{ass:router-coarea-regularity} ensures $F_{k\ell}$ is finite near zero and continuous at zero. Therefore
		\[
		\int_{-\delta}^{\delta} F_{k\ell}(t)\,\mathrm{d}t
		=2\delta F_{k\ell}(0)+\int_{-\delta}^{\delta}\bigl(F_{k\ell}(t)-F_{k\ell}(0)\bigr)\,\mathrm{d}t
		=2\delta F_{k\ell}(0)+o(\delta),
		\]
		which yields the desired expansion.
	\end{proof}
	
	Lemma~\ref{lem:tube-single} is purely geometric. It says that the probability of falling in a slab of width $\delta$ around a smooth hypersurface is proportional, to first order, to $\delta$ times a surface integral determined by $p_X$ and the local slope of the defining function.
	
	For $K>2$ we can obtain a linear upper bound on the pairwise profile by summing the pairwise contributions. Using the events $A_{k\ell}(\varepsilon,\tau)$ from~\eqref{eq:pairwise-ambiguity-event}, we have
	\[
	\mathrm{BM}^{\mathrm{pair}}_\phi(\kappa_\varepsilon\tau)
	= \PP\Bigl(\bigcup_{1\le k<\ell\le K} A_{k\ell}(\varepsilon,\tau)\Bigr)
	\le \sum_{1\le k<\ell\le K}
	\PP\bigl(A_{k\ell}(\varepsilon,\tau)\bigr).
	\]
	Each probability $\PP(A_{k\ell}(\varepsilon,\tau))$ has exactly the same
	representation as in the binary case with $S_{12}$ replaced by $S_{k\ell}$ and hence
	satisfies a linear expansion
	\[
	\PP\bigl(A_{k\ell}(\varepsilon,\tau)\bigr)
	= C_{\varepsilon,k\ell}\tau + o(\tau),
	\]
	with
	\[
	C_{\varepsilon,k\ell}
	= 2\log\!\Bigl(\frac{1-\varepsilon}{\varepsilon}\Bigr)
	\int_{\cS_{k\ell}}
	\frac{p_X(x)}{\norm{\nabla S_{k\ell}(x)}}
	\,\mathrm{d}\mathcal{H}^{d-1}(x),
	\]
	by the same coarea-formula argument. Summing over all pairs we obtain
	\begin{equation}\label{eq:BM-pair-linear}
		\mathrm{BM}^{\mathrm{pair}}_\phi(\kappa_\varepsilon\tau)
		\le \widetilde C_\varepsilon \tau + o(\tau),
		\qquad
		\widetilde C_\varepsilon
		:= \sum_{1\le k<\ell\le K} C_{\varepsilon,k\ell}.
	\end{equation}
	Thus, for general $K$, the leading constant is controlled by the total surface area
	of all pairwise decision boundaries weighted by $p_X/\|\nabla S_{k\ell}\|$.
	
	For more than two experts, the theorem chain below only needs a uniform upper bound rather than a sharp asymptotic. The following proposition packages the pairwise slab estimates into a single margin-tail estimate.
	
	\begin{proposition}[Global margin tail bound]\label{prop:margin-tail-upper}
		Suppose Assumption~\ref{ass:router-coarea-regularity} holds. Then there exist
		constants $\delta_0>0$ and $C_{\mathrm{margin}}>0$ such that
		\[
		\PP\bigl(\Delta_{\min}(X;\phi)\le\delta\bigr)
		\le C_{\mathrm{margin}}\;\delta,
		\qquad 0<\delta\le\delta_0,
		\]
		where $\Delta_{\min}$ is defined in~\eqref{eq:global-margin}.
	\end{proposition}
	
	\begin{proof}
		By definition of $\Delta_{\min}$ and a union bound,
		\[
		\PP\bigl(\Delta_{\min}(X;\phi)\le w\bigr)
		\le \sum_{1\le k<\ell\le K}
		\PP\bigl(|S_{k\ell}(X)|\le w\bigr).
		\]
		Fix a pair $(k,\ell)$. By the coarea formula as in
		Lemma~\ref{lem:tube-single}, for $w$ sufficiently small we have
		\[
		\PP\bigl(|S_{k\ell}(X)|\le w\bigr)
		= \int_{-w}^{w} F_{k\ell}(t)\,\mathrm{d}t,
		\]
		where $F_{k\ell}(t)$ is finite and continuous on a neighborhood of $0$ by
		Assumption~\ref{ass:router-coarea-regularity}. Hence there exists $w_0>0$ such
		that $\sup_{|t|\le w_0} F_{k\ell}(t) < \infty$, and for all
		$w\in(0,w_0]$,
		\[
		\PP\bigl(|S_{k\ell}(X)|\le w\bigr)
		\le 2w \sup_{|t|\le w_0} F_{k\ell}(t)
		=: C_{k\ell}\,w.
		\]
		Summing over all $\binom{K}{2}$ pairs gives
		\[
		\PP\bigl(\Delta_{\min}(X;\phi)\le w\bigr)
		\le \Bigl(\sum_{1\le k<\ell\le K} C_{k\ell}\Bigr)w,
		\]
		so we can take
		$C_{\mathrm{margin}} := \sum_{1\le k<\ell\le K} C_{k\ell}$.
	\end{proof}
	
	Proposition~\ref{prop:margin-tail-upper} says that, under mild regularity, the global margin distribution puts at most linear mass near zero. So inputs that sit close to a routing boundary become rare at rate $O(\delta)$ as the threshold $\delta$ shrinks. Later we will use exactly this fact to treat the $O(\tau)$ boundary layer as a perturbative part of the risk.
	
	The following elementary example makes the constant in the coarea expansion explicit. It is also a useful local model for a smooth router, because near a regular interface the score difference is well approximated by its first-order linearization.
	
	\begin{example}[Binary linear router under a Gaussian input]\label{ex:gaussian-linear-router}
		Let $K=2$, let $\cX=\RR^d$, and suppose $X\sim \mathcal{N}(0,\Sigma)$ with $\Sigma$ positive definite. Consider a linear score difference
		\begin{equation}\label{eq:gaussian-linear-score}
			S_{12}(x;\phi)=\langle \nu,x\rangle+b,
			\qquad \nu\in\RR^d\setminus\{0\},\quad b\in\RR.
		\end{equation}
		Set $\sigma_\nu^2:=\nu^\top\Sigma \nu$ and write $Z:=S_{12}(X;\phi)$. Then $Z\sim N(b,\sigma_\nu^2)$. If $\Phi_{\mathrm N}$ and $\varphi_{\mathrm N}$ denote the standard normal distribution function and density, respectively, then for every $r>0$,
		\begin{equation}\label{eq:gaussian-linear-slab-exact}
			\PP(|S_{12}(X;\phi)|\le r)
			=
			\Phi_{\mathrm N}\!\left(\frac{r-b}{\sigma_\nu}\right)
			-
			\Phi_{\mathrm N}\!\left(\frac{-r-b}{\sigma_\nu}\right).
		\end{equation}
		Consequently, as $r\downarrow0$,
		\begin{equation}\label{eq:gaussian-linear-slab-asymptotic}
			\PP(|S_{12}(X;\phi)|\le r)
			=
			2r\,\frac{\varphi_{\mathrm N}(b/\sigma_\nu)}{\sigma_\nu}
			+o(r).
		\end{equation}
		Taking $r=\kappa_\varepsilon\tau$ gives the binary boundary-mass asymptotic
		\begin{equation}\label{eq:gaussian-linear-bm-asymptotic}
			\mathrm{BM}^{\mathrm{pair}}_\phi(\kappa_\varepsilon\tau)
			=
			2\kappa_\varepsilon\,
			\frac{\varphi_{\mathrm N}(b/\sigma_\nu)}{\sigma_\nu}\,\tau
			+o(\tau),
			\qquad \tau\downarrow0.
		\end{equation}
		This agrees exactly with Proposition~\ref{prop:coarea}. Indeed, the decision surface is the hyperplane $\{x:\langle \nu,x\rangle+b=0\}$, $\nabla S_{12}=\nu$, and the coarea identity gives
		\begin{equation}\label{eq:gaussian-linear-coarea-constant}
			\int_{\{S_{12}=0\}}
			\frac{p_X(x)}{\|\nu\|}
			\,\mathrm d\mathcal H^{d-1}(x)
			=
			\frac{\varphi_{\mathrm N}(b/\sigma_\nu)}{\sigma_\nu}.
		\end{equation}
		Thus the abstract surface integral is simply the density at zero of the projected score $S_{12}(X;\phi)$. The example also shows how the coefficient changes with the position and scale of the router, since moving the hyperplane into a low-density region or scaling the logit difference upward reduces the amount of mass in an $O(\tau)$ score slab.
	\end{example}
	
	For the main results, the linear upper bound in Proposition~\ref{prop:margin-tail-upper} is the only geometric input that is needed. With stronger transversality assumptions one can derive a full first-order asymptotic for the multi-class boundary mass, including corrections for intersections of interfaces. That refinement would not change the soft-to-hard theorem below, so we do not pursue it here.
	
	\section{The variational zero-temperature limit and \texorpdfstring{$\Gamma$}{Gamma}-convergence}\label{sec:gamma}
	
	The preceding section controls boundary mass for a fixed router. We now turn that geometric estimate into a statement about objectives. Pointwise convergence of $L_\tau(\theta,\phi)$ to $L_0(\theta,\phi)$ is not enough for optimization: one wants to know what happens to minimizers, or almost minimizers, as the temperature is lowered. $\Gamma$-convergence is the standard framework for this passage from approximation of functionals to convergence of minimization problems \citep{Braides2002}. In the present setting it follows from a uniform soft-to-hard comparison and compactness of the parameter space.
	
	Recall the logit map, softmax weights, soft predictor, and soft risk from
	\eqref{eq:logit-map}--\eqref{eq:soft-risk-def}. For convenience we keep the label
	\begin{equation}\label{eq:Ltau}
		L_\tau(\theta,\phi)
		= \EE\bigl[(Y-h_{\theta,\phi,\tau}(X))^2\bigr]
	\end{equation}
	for the temperature-$\tau$ risk functional.
	
	\subsection{Hard routing and limiting risk}
	
	Given router scores $a(\cdot;\phi)$, the hard winner $k^\star$, hard predictor $h_{\theta,\phi,0}$, and hard risk $L_0$ are defined in
	\eqref{eq:hard-winner}--\eqref{eq:hard-risk-def}. Equivalently, the hard routing weights are
	\begin{equation}\label{eq:hard-router-weights}
		p_k^{(0)}(x;\phi):=\1_{k = k^\star(x;\phi)},
		\qquad k=1,\dots,K.
	\end{equation}
	Under the tail/margin condition below we will have $P_X(\text{tie})=0$, so the deterministic tie-breaking convention in~\eqref{eq:hard-winner} does not affect the risk. For later reference we also write
	\begin{equation}\label{eq:L0}
		L_0(\theta,\phi)
		= \EE\bigl[(Y-h_{\theta,\phi,0}(X))^2\bigr].
	\end{equation}
	
	The section has two steps. First we prove a uniform quantitative estimate,
		\[
		\sup_{\theta,\phi}\abs{L_\tau(\theta,\phi)-L_0(\theta,\phi)}\lesssim\tau^\alpha,
		\]
		under a margin-tail condition. We then use this estimate to obtain $\Gamma$-convergence of $L_\tau$ to $L_0$ as $\tau\to0$.
	
	\subsection{Softmax tail bounds}
	
	The only softmax fact needed for the comparison is the following elementary tail bound. It says that once the top logit is separated from the runner-up, the total mass assigned to nonwinning experts is exponentially small in the margin divided by temperature.
	
	\begin{lemma}[Softmax tail, pointwise]\label{lem:softmax-tail}
		Let $z=(z_1,\dots,z_K)\in\RR^K$. Let $z_{(1)}\ge\cdots\ge z_{(K)}$ be the
		sorted values and let $p_k(z;\tau)$ be the softmax weights with
		temperature~$\tau$, i.e.\ $p_k(z;\tau)=\exp(z_k/\tau)/\sum_i \exp(z_i/\tau)$.
		Then for any $j\ge2$,
		\[
		p_{(j)}(z;\tau)
		\le \exp\bigl( -(z_{(1)}-z_{(j)})/\tau \bigr).
		\]
		In particular,
		\[
		\sum_{j=2}^K p_{(j)}(z;\tau)
		\le (K-1)\exp\bigl( -\Delta(z)/\tau \bigr),
		\]
		where $\Delta(z):=z_{(1)}-z_{(2)}$ is the top-two margin.
	\end{lemma}
	
	\begin{proof}
		For each $j\ge2$,
		\[
		p_{(j)}(z;\tau)
		= \frac{\exp(z_{(j)}/\tau)}{\sum_{i=1}^K\exp(z_{(i)}/\tau)}
		\le \frac{\exp(z_{(j)}/\tau)}{\exp(z_{(1)}/\tau)}
		= \exp\bigl(-(z_{(1)}-z_{(j)})/\tau\bigr).
		\]
		Summing over $j\ge2$ yields the second claim with $\Delta(z)$ defined as
		above.
	\end{proof}
	
	\subsection{A quantitative soft--hard comparison}
	
	Recall the top--two margin $\Delta(x;\phi)$ from~\eqref{eq:routing-margin}. A positive margin $\Delta(x;\phi)>0$ is equivalent to uniqueness of the argmax of $a(x;\phi)$ at~$x$.
	The rate comes from a uniform margin-tail condition. Informally, it rules out parameter values for which a non-negligible amount of probability mass remains arbitrarily close to a routing tie.
	
	\begin{assumption}[Uniform margin-tail control]\label{ass:uniform-tail}
		There exist constants $C_{\mathrm{mt}}>0$ and $\alpha>0$ such that, for all admissible router parameters $\phi$ and all $t>0$,
		\begin{equation}\label{eq:uniform-margin-tail}
			\PP\bigl(\Delta(X;\phi) \le t\bigr)
			\le C_{\mathrm{mt}} t^\alpha.
		\end{equation}
	\end{assumption}
	
	Assumption~\ref{ass:uniform-tail} gives uniform control of the margin distribution near zero. This is the only distributional input needed for the bounded-output soft--hard theorem below. If outputs are unbounded, a corresponding result would require additional uniform integrability or localized moment assumptions; we do not use such a variant in the main theorem.
	
	To obtain a uniform rate, we impose a global boundedness assumption on the response and expert families. This is not conceptually central; it is simply a convenient way to keep the soft--hard comparison from carrying additional tail terms.
	
	\begin{assumption}[Uniform boundedness]\label{ass:bounded-Y-f}
		There exist finite constants $B_Y,B_f>0$ such that
		$\abs{Y}\le B_Y$ almost surely and
		$\abs{f_k(x;\theta_k)}\le B_f$ for all $k$, all $x\in\cX$ and all admissible
		$\theta_k$.
	\end{assumption}
	
	Assumption~\ref{ass:bounded-Y-f} is included for convenience. It lets the quantitative soft--hard bounds go through without carrying a separate layer of tail bookkeeping. The condition is natural in classification after label smoothing or clipping, and in regression when labels and predictions are truncated to a fixed range. If the outputs are unbounded, one can usually replace it by moment assumptions plus concentration estimates, but the constants and the proof become noticeably less clean.
	Before stating the uniform theorem, it is useful to record the direct fixed-router consequence of the coarea control from Section~\ref{sec:global-geometry}. Under Assumption~\ref{ass:router-coarea-regularity}, the near-tie layer has linear mass, and that already forces an explicit linear soft-to-hard risk bound for bounded outputs.
	
	\begin{corollary}[Coarea-based $O(\tau)$ soft-to-hard bound]\label{cor:coarea-soft-hard-linear}
		Suppose Assumption~\ref{ass:router-coarea-regularity} holds for a fixed router parameter $\phi$, the argmax of $a(\cdot;\phi)$ is unique $P_X$-almost surely, and Assumption~\ref{ass:bounded-Y-f} holds. Then there exist constants $\tau_0>0$ and $C_{\mathrm{co}}(\phi)>0$ such that for every admissible expert parameter $\theta$ and every $\tau\in(0,\tau_0]$,
		\[
		\abs{L_\tau(\theta,\phi)-L_0(\theta,\phi)}
		\le C_{\mathrm{co}}(\phi)\,\tau.
		\]
	\end{corollary}
	
	\begin{proof}
		Let $\Delta(x;\phi)=a_{(1)}(x;\phi)-a_{(2)}(x;\phi)$ denote the top-two margin. Since $\Delta_{\min}(x;\phi)\le \Delta(x;\phi)$ for every $x$, the event $\{\Delta(X;\phi)\le t\}$ is contained in $\{\Delta_{\min}(X;\phi)\le t\}$. Proposition~\ref{prop:margin-tail-upper} therefore yields constants $\delta_0>0$ and $C_{\mathrm{margin}}(\phi)>0$ such that
		\[
		\PP\bigl(\Delta(X;\phi)\le t\bigr)\le C_{\mathrm{margin}}(\phi)\,t,
		\qquad 0<t\le \delta_0.
		\]
		Fix $\theta$ and abbreviate $h_\tau:=h_{\theta,\phi,\tau}$ and $h_0:=h_{\theta,\phi,0}$. Using
		\[
		(Y-h_\tau(X))^2-(Y-h_0(X))^2=(h_0(X)-h_\tau(X))(2Y-h_\tau(X)-h_0(X)),
		\]
		together with the bounds $\abs{Y}\le B_Y$ and $\abs{h_\tau(X)},\abs{h_0(X)}\le B_f$, we obtain
		\[
		\abs{L_\tau(\theta,\phi)-L_0(\theta,\phi)}
		\le (2B_Y+2B_f)\,\EE\bigl[\abs{h_\tau(X)-h_0(X)}\bigr].
		\]
		The softmax tail bound then gives
		\[
		\abs{h_\tau(x)-h_0(x)}
		\le 2B_f(K-1)e^{-\Delta(x;\phi)/\tau}.
		\]
		Hence
		\[
		\abs{L_\tau(\theta,\phi)-L_0(\theta,\phi)}
		\le 4B_f(B_Y+B_f)(K-1)\,\EE\bigl[e^{-\Delta(X;\phi)/\tau}\bigr].
		\]
		Using integration by parts with the distribution function of $\Delta(X;\phi)$,
		\[
		\EE\bigl[e^{-\Delta(X;\phi)/\tau}\bigr]
		=\frac{1}{\tau}\int_0^\infty \PP\bigl(\Delta(X;\phi)\le t\bigr)e^{-t/\tau}\,dt.
		\]
		Split the integral at $\delta_0$. On $(0,\delta_0]$ we use the linear bound above, while on $[\delta_0,\infty)$ we use $\PP(\Delta\le t)\le 1$. This gives
		\[
		\EE\bigl[e^{-\Delta(X;\phi)/\tau}\bigr]
		\le \frac{C_{\mathrm{margin}}(\phi)}{\tau}\int_0^{\delta_0} t e^{-t/\tau}\,dt + e^{-\delta_0/\tau}
		\le C_{\mathrm{margin}}(\phi)\tau + e^{-\delta_0/\tau}.
		\]
		Choose $\tau_0>0$ so that $e^{-\delta_0/\tau}\le \tau$ for all $\tau\in(0,\tau_0]$. Then
		\[
		\EE\bigl[e^{-\Delta(X;\phi)/\tau}\bigr]\le \bigl(C_{\mathrm{margin}}(\phi)+1\bigr)\tau,
		\qquad 0<\tau\le \tau_0,
		\]
		and the claimed bound follows after absorbing constants into $C_{\mathrm{co}}(\phi)$.
	\end{proof}
	
	We can now state the quantitative comparison. The theorem says that the soft objective approaches the hard objective at the same polynomial order as the margin-tail bound. Thus the rate is not an intrinsic property of softmax alone; it is set by how much input mass the router places near ties.
	
	\begin{theorem}[Uniform soft-to-hard approximation]\label{thm:soft-hard-quant}
		Suppose Assumptions~\ref{ass:uniform-tail} and~\ref{ass:bounded-Y-f} hold. Since the margin-tail condition
		implies that the hard argmax is unique $P_X$-almost surely, the hard router is
		well defined up to a null set. Then there exists a constant
		$C_{\mathrm{sh}}>0$, depending only on $K$, $C_{\mathrm{mt}}$, $\alpha$, $B_Y$, and $B_f$, such that
		for all $\tau\in(0,1]$ and all $(\theta,\phi)$,
		\[
		\abs{L_\tau(\theta,\phi)-L_0(\theta,\phi)}
		\;\le\;
		C_{\mathrm{sh}}\;\tau^{\alpha}.
		\]
		In particular, $L_\tau\to L_0$ uniformly on compact parameter sets, with rate $O(\tau^{\alpha})$.
	\end{theorem}
	
	\begin{proof}
		We argue in three short steps. First we reduce the risk difference to the predictor gap $h_\tau-h_0$. Then we control that gap with the softmax tail. Finally we integrate against the margin distribution.
		
		We first reduce the risk difference to an $L^1$ predictor gap. Fix $(\theta,\phi)$ and abbreviate $h_\tau:=h_{\theta,\phi,\tau}$ and $h_0:=h_{\theta,\phi,0}$. Then
		\[
		(Y-h_\tau(X))^2 - (Y-h_0(X))^2
		= (h_0(X)-h_\tau(X))\bigl(2Y-h_\tau(X)-h_0(X)\bigr).
		\]
		Taking absolute values and expectations gives
		\begin{equation}\label{eq:soft-hard-start}
			\abs{L_\tau(\theta,\phi)-L_0(\theta,\phi)}
			\le \EE\Bigl[
			\abs{h_\tau(X)-h_0(X)}
			\bigl(2\abs{Y}+\abs{h_\tau(X)}+\abs{h_0(X)}\bigr)
			\Bigr].
		\end{equation}
		By Assumption~\ref{ass:bounded-Y-f}, $|Y|\le B_Y$ and $|h_\tau(X)|,|h_0(X)|\le B_f$ almost surely. Hence
		\[
		2\abs{Y}+\abs{h_\tau(X)}+\abs{h_0(X)}
		\le 2B_Y+2B_f
		=: C_0,
		\]
		and therefore
		\begin{equation}\label{eq:soft-hard-L1}
			\abs{L_\tau(\theta,\phi)-L_0(\theta,\phi)}
			\le C_0\;\EE\bigl[\abs{h_\tau(X)-h_0(X)}\bigr].
		\end{equation}
		
		We next bound $h_\tau-h_0$ through the score margin. For each $x$, let $k^\star(x;\phi)$ be the unique maximizer of $a_k(x;\phi)$; this is defined $P_X$-almost surely by assumption. Then
		\[
		h_\tau(x)=\sum_{k=1}^K p_k^{(\tau)}(x;\phi)f_k(x;\theta_k),
		\qquad
		h_0(x)=f_{k^\star(x;\phi)}(x;\theta_{k^\star(x;\phi)}).
		\]
		Subtracting and using $\sum_k p_k^{(\tau)}(x;\phi)=1$ gives
		\[
		h_\tau(x)-h_0(x)
		= \sum_{k\neq k^\star(x;\phi)} p_k^{(\tau)}(x;\phi)
		\bigl(f_k(x;\theta_k)-f_{k^\star(x;\phi)}(x;\theta_{k^\star(x;\phi)})\bigr).
		\]
		Since $|f_k|\le B_f$, we obtain the pointwise estimate
		\begin{equation}\label{eq:h-soft-hard-pointwise}
			\abs{h_\tau(x)-h_0(x)}
			\le 2B_f\sum_{k\neq k^\star(x;\phi)} p_k^{(\tau)}(x;\phi)
			= 2B_f\bigl(1-p_{k^\star(x;\phi)}^{(\tau)}(x;\phi)\bigr).
		\end{equation}
		Now sort the logits at $x$ as $z_{(1)}\ge z_{(2)}\ge\cdots\ge z_{(K)}$. By Lemma~\ref{lem:softmax-tail},
		\[
		1-p_{k^\star(x;\phi)}^{(\tau)}(x;\phi)
		\le (K-1)\exp\bigl(-(z_{(1)}-z_{(2)})/\tau\bigr)
		= (K-1)\exp\bigl(-\Delta(x;\phi)/\tau\bigr).
		\]
		Combining this with \eqref{eq:h-soft-hard-pointwise},
		\[
		\abs{h_\tau(x)-h_0(x)}
		\le 2B_f(K-1)\exp\bigl(-\Delta(x;\phi)/\tau\bigr).
		\]
		After taking expectations we arrive at
		\begin{equation}\label{eq:Eh-gap}
			\EE\bigl[\abs{h_\tau(X)-h_0(X)}\bigr]
			\le 2B_f(K-1)\EE\bigl[\exp\bigl(-\Delta(X;\phi)/\tau\bigr)\bigr].
		\end{equation}
		
		We now integrate the margin-tail bound. Let $F_\phi(t):=\PP(\Delta(X;\phi)\le t)$. By the margin-tail bound~\eqref{eq:uniform-margin-tail},
		\[
		F_\phi(t)\le C_{\mathrm{mt}}t^\alpha,
		\qquad t>0,
		\]
		uniformly in $\phi$. Since $F_\phi(0)=\PP(\Delta=0)=0$, integration by parts yields
		\begin{align*}
			\EE\bigl[e^{-\Delta(X;\phi)/\tau}\bigr]
			&= \int_0^\infty e^{-t/\tau}\,dF_\phi(t) \\
			&= \frac{1}{\tau}\int_0^\infty F_\phi(t)e^{-t/\tau}\,dt \\
			&\le \frac{C_{\mathrm{mt}}}{\tau}\int_0^\infty t^\alpha e^{-t/\tau}\,dt.
		\end{align*}
		With the substitution $t=\tau u$ we obtain
		\[
		\EE\bigl[e^{-\Delta(X;\phi)/\tau}\bigr]
		\le C_{\mathrm{mt}}\tau^\alpha\int_0^\infty u^\alpha e^{-u}\,du
		= C_{\mathrm{mt}}\Gamma(\alpha+1)\tau^\alpha.
		\]
		Here $\Gamma(\alpha+1)$ denotes the usual Gamma function. Thus there is a constant $C_1>0$, depending only on $C_{\mathrm{mt}}$ and $\alpha$, such that
		\[
		\EE\bigl[e^{-\Delta(X;\phi)/\tau}\bigr]\le C_1\tau^\alpha
		\qquad\text{for all }\tau\in(0,1]\text{ and all }\phi.
		\]
		Plugging this into \eqref{eq:Eh-gap} and then into \eqref{eq:soft-hard-L1} gives
		\[
		\abs{L_\tau(\theta,\phi)-L_0(\theta,\phi)}
		\le C_{\mathrm{sh}}\tau^\alpha,
		\]
		where $C_{\mathrm{sh}}$ depends only on $B_Y,B_f,K,C_{\mathrm{mt}}$, and $\alpha$. The constants are uniform on compact parameter sets, so the bound is uniform in $(\theta,\phi)$.
	\end{proof}
	
	\begin{remark}[Relation to \texorpdfstring{$\Gamma$}{Gamma}-convergence]
		Theorem~\ref{thm:soft-hard-quant} does more than provide the pointwise convergence $L_\tau\to L_0$ needed later for $\Gamma$-convergence. It also gives a uniform polynomial rate on compact parameter sets. That extra uniformity is useful whenever one wants to control the soft--hard approximation over an entire family of parameters, for instance along an annealing trajectory.
	\end{remark}
	
	\subsection{\texorpdfstring{$\Gamma$}{Gamma}-convergence}
	
	We equip the parameter space with a product topology that makes
	$(\theta,\phi)\mapsto f_k(\cdot;\theta_k)$ continuous in $L^2(P_X)$; for
	instance, if $\Theta_k$ is compact and the map $\theta_k\mapsto f_k$ is
	continuous in $L^2(P_X)$ for each~$k$, then the product topology on
	$\Theta$ suffices. For the router parameters~$\phi$ we assume
	a topology such that $(\theta,\phi)\mapsto a(\cdot;\phi)$ is continuous in
	$L^2(P_X)$ and $L^\infty(P_X)$.

	The quantitative soft--hard estimate from Theorem~\ref{thm:soft-hard-quant} is already uniform on compact parameter sets under the margin-tail bound~\eqref{eq:uniform-margin-tail} and Assumption~\ref{ass:bounded-Y-f}. We therefore do not need a separate uniform boundary-mass assumption for the $\Gamma$-convergence theorem below.
	
	Under these conditions, the limiting statement can be phrased variationally. This is the right level of generality for optimization: it says not only that the values of the soft functionals converge, but also that their almost minimizers select the hard-routing problem in the limit. Since the margin-tail bound~\eqref{eq:uniform-margin-tail} gives $P_X(\Delta(X;\phi)\le t)\le C_{\mathrm{mt}}t^\alpha$ for every $t>0$, it also implies $P_X(\Delta(X;\phi)=0)=0$ after letting $t\downarrow0$. Thus the hard router is uniquely defined $P_X$-almost surely in the theorem below.
	
	\begin{theorem}[Variational zero-temperature limit]\label{thm:gamma}
		Suppose Assumptions~\ref{ass:uniform-tail} and~\ref{ass:bounded-Y-f} hold. Assume also that the parameter
		spaces are compact in the topology described above. Then the functionals
		$L_\tau$ $\Gamma$-converge to $L_0$ on
		the parameter space as $\tau\to0$. In particular, the following statements hold.
		\begin{enumerate}[label=(\roman*), leftmargin=*]
			\item (\emph{$\Gamma$-liminf inequality}) If $(\theta_\tau,\phi_\tau)$
			converges to $(\theta,\phi)$ as $\tau\to0$, then
			\[
			L_0(\theta,\phi)
			\le \liminf_{\tau\to0} L_\tau(\theta_\tau,\phi_\tau).
			\]
			\item (\emph{Recovery sequence}) For every $(\theta,\phi)$ there exists a
			sequence $(\theta_\tau,\phi_\tau)\to(\theta,\phi)$ such that
			\[
			L_0(\theta,\phi)
			\ge \limsup_{\tau\to0} L_\tau(\theta_\tau,\phi_\tau).
			\]
			\item Any sequence of almost minimizers $(\theta_\tau,\phi_\tau)$ with
			$L_\tau(\theta_\tau,\phi_\tau)-\inf L_\tau\to0$ is precompact, and every
			cluster point is a minimizer of~$L_0$.
		\end{enumerate}
	\end{theorem}

	\begin{proof}
		The proof combines the uniform soft-to-hard approximation with continuity of the hard functional. Once those two ingredients are in place, the $\Gamma$-liminf inequality, the recovery sequence, and the statement about almost minimizers follow directly.
		
		We start from the uniform convergence of $L_\tau$ to $L_0$.
		Theorem~\ref{thm:soft-hard-quant} already gives a uniform modulus of convergence on compact parameter sets. In particular, there exists a deterministic function $\omega(\tau)\downarrow0$ such that
		\begin{equation}\label{eq:gamma-uniform-modulus}
			\sup_{(\theta,\phi)} |L_\tau(\theta,\phi)-L_0(\theta,\phi)|
			\le \omega(\tau).
		\end{equation}
		For the present theorem one may take $\omega(\tau)=C\tau^\alpha$ from Theorem~\ref{thm:soft-hard-quant}.
		
		It remains to check continuity of the hard functional $L_0$ in the chosen topology.
		Let $(\theta_n,\phi_n)\to(\theta,\phi)$ in the topology described before the theorem. Write
		\[
		h_n(x):=h_{\theta_n,\phi_n,0}(x),
		\qquad
		h(x):=h_{\theta,\phi,0}(x).
		\]
		We claim that $h_n\to h$ in $L^1(P_X)$, which will imply $L_0(\theta_n,\phi_n)\to L_0(\theta,\phi)$.
		
		Set
		\[
		\delta_n:=\max_{1\le k\le K}\|a_k(\cdot;\phi_n)-a_k(\cdot;\phi)\|_{L^\infty(P_X)}\xrightarrow[n\to\infty]{}0.
		\]
		Define the stable region
		\[
		B_n:=\{x:\Delta(x;\phi)>2\delta_n\}.
		\]
		On $B_n$, the maximizer of the logits is unchanged when $\phi$ is replaced by $\phi_n$. Indeed, if $k^\star(x;\phi)$ is the unique winner for $\phi$, then every competing logit remains at least $\Delta(x;\phi)-2\delta_n>0$ below it under the perturbation. Hence on $B_n$ the hard router selects the same expert for $h_n$ and $h$, and therefore
		\[
		|h_n(x)-h(x)|
		\le \sum_{k=1}^K |f_k(x;\theta_{n,k})-f_k(x;\theta_k)|\,\mathbf 1_{\{k^\star(x;\phi)=k\}}.
		\]
		Taking expectations and using $L^2(P_X)$-continuity of each expert map,
		\begin{equation}\label{eq:hard-cont-stable}
			\EE\bigl[|h_n-h|\mathbf 1_{B_n}\bigr]
			\le \sum_{k=1}^K \|f_k(\cdot;\theta_{n,k})-f_k(\cdot;\theta_k)\|_{L^1(P_X)}
			\xrightarrow[n\to\infty]{}0.
		\end{equation}
		On the complement $B_n^c=\{\Delta(X;\phi)\le 2\delta_n\}$, the hard predictions may differ because the winner can switch, but this set is small. Indeed, the margin-tail bound~\eqref{eq:uniform-margin-tail} gives
		\[
		P_X(B_n^c)\le C_{\mathrm{mt}}\,(2\delta_n)^\alpha
		\]
		for all large $n$, and $|h_n|,|h|\le B_f$. Therefore
		\begin{equation}\label{eq:hard-cont-badset}
			\EE\bigl[|h_n-h|\mathbf 1_{B_n^c}\bigr]
			\le 2B_f\,P_X(B_n^c)
			\le 2^{\alpha+1}B_f C_{\mathrm{mt}}\,\delta_n^\alpha
			\xrightarrow[n\to\infty]{}0.
		\end{equation}
		Combining \eqref{eq:hard-cont-stable} and \eqref{eq:hard-cont-badset} yields $h_n\to h$ in $L^1(P_X)$. Since $|Y|\le B_Y$ and $|h_n|,|h|\le B_f$,
		\[
		\bigl|(Y-h_n)^2-(Y-h)^2\bigr|
		\le 2(B_Y+B_f)|h_n-h|,
		\]
		so $L_0(\theta_n,\phi_n)\to L_0(\theta,\phi)$. Thus $L_0$ is continuous.
		
		We can now verify $\Gamma$-convergence and identify the behavior of almost minimizers. Let $(\theta_\tau,\phi_\tau)\to(\theta,\phi)$. From \eqref{eq:gamma-uniform-modulus},
		\[
		L_\tau(\theta_\tau,\phi_\tau)
		\ge L_0(\theta_\tau,\phi_\tau)-\omega(\tau).
		\]
		Taking $\liminf$ and using continuity of $L_0$ gives
		\[
		\liminf_{\tau\downarrow0}L_\tau(\theta_\tau,\phi_\tau)
		\ge \liminf_{\tau\downarrow0}L_0(\theta_\tau,\phi_\tau)
		= L_0(\theta,\phi),
		\]
		which is the $\Gamma$-liminf inequality.
		
		For the recovery sequence, fix $(\theta,\phi)$ and take the constant sequence $(\theta_\tau,\phi_\tau)\equiv(\theta,\phi)$. Then \eqref{eq:gamma-uniform-modulus} gives
		\[
		\limsup_{\tau\downarrow0}L_\tau(\theta_\tau,\phi_\tau)
		\le L_0(\theta,\phi)+\limsup_{\tau\downarrow0}\omega(\tau)
		= L_0(\theta,\phi).
		\]
		This proves the recovery property.
		
		Finally, let $(\theta_\tau,\phi_\tau)$ be almost minimizers with $L_\tau(\theta_\tau,\phi_\tau)-\inf L_\tau\to0$. Compactness of the parameter space gives precompactness. Let $(\theta_{\tau_j},\phi_{\tau_j})\to(\theta,\phi)$ along a subsequence. By the liminf inequality,
		\[
		L_0(\theta,\phi)
		\le \liminf_{j\to\infty}L_{\tau_j}(\theta_{\tau_j},\phi_{\tau_j}).
		\]
		For any fixed $(\bar\theta,\bar\phi)$,
		\[
		L_{\tau_j}(\theta_{\tau_j},\phi_{\tau_j})
		\le L_{\tau_j}(\bar\theta,\bar\phi)+o(1).
		\]
		Passing to the limit and using \eqref{eq:gamma-uniform-modulus} again gives
		\[
		L_0(\theta,\phi)\le L_0(\bar\theta,\bar\phi).
		\]
		Since $(\bar\theta,\bar\phi)$ was arbitrary, $(\theta,\phi)$ is a minimizer of $L_0$.
	\end{proof}

	\section{A conditional landscape-transfer principle at small temperature}\label{sec:benign}
	
	The boundary-layer estimates do not, by themselves, make the MoE landscape benign. They explain how a soft objective approaches a hard-routing objective. Whether the limiting hard problem has identifiable solutions, negative curvature away from those solutions, or useful local geometry is a separate question. This section makes that separation explicit. We assume the relevant hard-routing geometry and then prove that it persists for sufficiently small temperature when first- and second-order soft-to-hard transfer holds near the critical region. The result is therefore a transfer principle, not a general benign-landscape theorem for MoE training.
	
	We use the same notation for router logits, softmax weights, and student prediction as in~\eqref{eq:logit-map}--\eqref{eq:moe-predictor}.
	
	We first specify the reference decomposition. The teacher consists of a measurable partition $(\Omega_k^\star)_{k=1}^K$ of $\cX$, understood up to $\PP$-null sets, and teacher experts $f_k^\star:\cX\to\RR$.
	Given $X\sim P_X$, the label is generated by
	\begin{equation}\label{eq:teacher}
		Y = \mu(X)+\xi,
		\qquad
		\EE[\xi\mid X]=0.
	\end{equation}
	The teacher signal is the piecewise expert function
	\begin{equation}\label{eq:teacher-mu}
		\mu(x):=\sum_{k=1}^K \1_{x\in\Omega_k^\star}\, f_k^\star(x)
		\qquad\text{(equivalently, }\mu(x)=f_k^\star(x)\text{ for }x\in\Omega_k^\star\text{).}
	\end{equation}
	
	\begin{assumption}[Standing teacher--student conditions]\label{ass:standing}
		The teacher model~\eqref{eq:teacher} holds with $\EE[\xi^2]<\infty$. The sets $(\Omega_k^\star)_{k=1}^K$ form a measurable partition of $\cX$ up to $P_X$-null sets, each $f_k^\star$ is measurable, and there exists $B<\infty$ such that $\abs{\mu(X)}\le B$ almost surely.
	\end{assumption}
	
	Assumption~\ref{ass:standing} fixes the reference decomposition used in the rest of the section. Its substantive role is to provide a partition and cellwise target experts, so that misrouting has a population-level meaning. Outside a literal teacher--student interpretation, the later statements involving $(\theta^\star,\phi^\star)$ can be read relative to any fixed reference decomposition satisfying the same assumptions.
	
	The student model must be rich enough to represent the teacher experts on their own cells, and its predictions must have enough integrability for the risk comparisons below. We record both requirements in one assumption because they are used together in the landscape-transfer argument.
	
	\begin{assumption}[Student realizability and boundedness/moments]\label{ass:student-moments}
		There exist parameters $\theta^\star=(\theta_1^\star,\dots,\theta_K^\star)$ such that
		$f_k(\cdot;\theta_k^\star)=f_k^\star(\cdot)$ $\PP$-a.s.\ on $\Omega_k^\star$. In addition, either the uniform-boundedness branch holds, meaning that there is $B_{\mathrm{stu}}<\infty$ with
		$\sup_{k,x}\abs{f_k(x;\theta_k)}\le B_{\mathrm{stu}}$ for all parameters under consideration, or the moment-control branch holds, meaning that $\sup_{k}\EE[\abs{f_k(X;\theta_k)}^4]<\infty$ and hence $\EE[h_{\theta,\phi,\tau}(X)^2]<\infty$ for all $(\theta,\phi,\tau)$.
	\end{assumption}
	
	Assumption~\ref{ass:student-moments} does two different jobs. One part says that the student class can realize the teacher experts on their own cells. The other part gives the boundedness or moment control needed for exchanging expectations and derivatives and for keeping the remainder terms finite. Uniform boundedness is natural for bounded-output experts or after explicit clipping, while moment assumptions are the right substitute when the outputs are unbounded but still have controlled tails.
	
	In applications this type of control is often enforced only indirectly, through regularization and clipping. If the training dynamics produce rare but extremely large losses or gradients, then one should expect the present boundary-mass bounds to become correspondingly crude.
	
	\begin{assumption}[Teacher-router realizability and interface regularity]\label{ass:teacher-router}
		There exists a router parameter $\phi^\star$ such that the hard partition induced by the logits $a_k(\cdot;\phi^\star)$ agrees with the teacher partition $(\Omega_k^\star)_{k=1}^K$ up to $P_X$-null sets. For each adjacent teacher pair $(k,\ell)$, the corresponding teacher interface
		\[
		\cS_{k\ell}^\star:=\partial\Omega_k^\star\cap\partial\Omega_\ell^\star
		\]
		is, away from multiway junctions (that is, points where three or more routing cells meet), a compact $C^2$ hypersurface on which $\nabla(a_k-a_\ell)(x;\phi^\star)\neq 0$. In addition, $p_X$ is bounded on a neighborhood of the union of the teacher interfaces, and the pairwise score differences $S_{k\ell}(\cdot;\phi^\star)$ satisfy the tubular coarea regularity of Assumption~\ref{ass:router-coarea-regularity} on the teacher-router interfaces.
	\end{assumption}
	
	Assumption~\ref{ass:teacher-router} supplies the reference router that is used later when comparing a global minimizer to the teacher configuration. It also gives the regularity needed to define the teacher-side normals and traces and, by explicitly invoking Assumption~\ref{ass:router-coarea-regularity} at $\phi^\star$, to apply the linear boundary-mass estimate at the teacher router.
	
	The student MoE uses the same number of experts~$K$, with parameters $(\theta,\phi)$ and temperature~$\tau$. Throughout this section $\tau$ should be thought of as small but positive. In the limiting picture, the router becomes hard and the model approaches a piecewise-expert model whose cell boundaries move with~$\phi$.
	
	\subsection{Risk decomposition: interior vs boundary tubes}
	
	The noise in the teacher model plays no role in the optimization geometry, so we first remove it from the notation. Conditioning on $X$ gives the decomposition
	\[
	L_\tau(\theta,\phi)
	= \EE\bigl[ (Y-h_{\theta,\phi,\tau}(X))^2 \bigr]
	= \EE\bigl[ (\mu(X)-h_{\theta,\phi,\tau}(X))^2 \bigr] + \EE[\xi^2].
	\]
	The noise variance $\EE[\xi^2]$ is independent of $(\theta,\phi,\tau)$, so we
	focus on
	\begin{equation}\label{eq:excess-risk-def}
	\widetilde L_\tau(\theta,\phi)
	:= \EE\bigl[ (\mu(X)-h_{\theta,\phi,\tau}(X))^2 \bigr].
	\end{equation}
	
	For a fixed $\varepsilon\in(0,1/2)$, router parameter $\phi$, and temperature $\tau>0$, define the \emph{$K$-way ambiguity region}
	\begin{equation}\label{eq:Kway-ambiguity-region}
		\mathcal{U}_\varepsilon(\tau;\phi)
		:= \Bigl\{x\in\cX:\ \max_{1\le k\le K} p_k^{(\tau)}(x;\phi)\le 1-\varepsilon\Bigr\},
		\qquad
		\mathcal{G}_\varepsilon(\tau;\phi):=\cX\setminus\mathcal{U}_\varepsilon(\tau;\phi).
	\end{equation}
	Thus on $\mathcal{G}_\varepsilon(\tau;\phi)$ there exists (at least one) expert with routing weight at least $1-\varepsilon$.
	
	For the risk decomposition we need the exact mass of this $K$-way ambiguity event, not merely a pairwise slab upper bound. We therefore define
	\begin{equation}\label{eq:BM-def-here}
		\mathrm{BM}^{\mathrm{amb}}_{\phi,\varepsilon}(\tau)
		:=\PP\bigl(X\in\mathcal{U}_\varepsilon(\tau;\phi)\bigr).
	\end{equation}
	
	\begin{remark}[Boundary-mass taxonomy]\label{rem:BM-benign-vs-geom}
		The paper uses three boundary-mass quantities, each for a different role. The top--two profile $\mathrm{BM}^{\mathrm{top}}_\phi(w)$ from~\eqref{eq:BM-top-profile} is tied to softmax concentration. The pairwise profile $\mathrm{BM}^{\mathrm{pair}}_\phi(w)$ from~\eqref{eq:BM-pair-profile} is the geometric slab quantity used in the coarea estimates; because $\Delta_{\min}\le\Delta$, this profile is an upper envelope for top--two near-ties and may also count near-ties between nonwinning experts. The ambiguity mass $\mathrm{BM}^{\mathrm{amb}}_{\phi,\varepsilon}(\tau)$ in~\eqref{eq:BM-def-here} is different again: it is the exact probability of the $K$-way softmax event used in the risk decomposition.
		
		Recall $\kappa_\varepsilon$ from~\eqref{eq:kappa-eps} and define
		\begin{equation}\label{eq:kappa-eps-K}
		\kappa_{\varepsilon,K}:=\log\!\Bigl(\frac{(K-1)(1-\varepsilon)}{\varepsilon}\Bigr).
		\end{equation}
		Then, for every $\tau>0$,
		\[
		\bigl\{\Delta(X;\phi)\le \kappa_\varepsilon\tau\bigr\}
		\subseteq \bigl\{X\in\mathcal{U}_\varepsilon(\tau;\phi)\bigr\}
		\subseteq \bigl\{\Delta(X;\phi)\le \kappa_{\varepsilon,K}\tau\bigr\}.
		\]
		Equivalently,
		\[
		\mathrm{BM}^{\mathrm{top}}_\phi(\kappa_\varepsilon\tau)
		\le \mathrm{BM}^{\mathrm{amb}}_{\phi,\varepsilon}(\tau)
		\le \mathrm{BM}^{\mathrm{top}}_\phi(\kappa_{\varepsilon,K}\tau)
		\le \mathrm{BM}^{\mathrm{pair}}_\phi(\kappa_{\varepsilon,K}\tau).
		\]
		Thus these symbols are not aliases. Top--two margins control softmax tails, pairwise slabs provide the geometric estimates, and $\mathcal U_\varepsilon(\tau;\phi)$ is the event used in the risk decomposition. In the binary case $K=2$, the three descriptions collapse to the same slab up to the common width $\kappa_\varepsilon\tau$.
	\end{remark}

	\begin{lemma}[Risk decomposition]\label{lem:risk-decomp}
	Under Assumption~\ref{ass:standing} and the uniform-boundedness branch of Assumption~\ref{ass:student-moments}, for any
	$\varepsilon\in(0,1/2)$,
	\[
	\widetilde L_\tau(\theta,\phi)
	= \widetilde L_{\tau,\mathrm{int}}(\theta,\phi)
	+ \widetilde L_{\tau,\mathrm{bdry}}(\theta,\phi),
	\]
	where the two terms are the unnormalized restrictions of the risk to the confident and ambiguous regions,
	\begin{align*}
		\widetilde L_{\tau,\mathrm{int}}(\theta,\phi)
		&:= \EE\bigl[(\mu(X)-h_{\theta,\phi,\tau}(X))^2
		\,\1_{\{X\in\mathcal{G}_\varepsilon(\tau;\phi)\}}\bigr],
		\\
		\widetilde L_{\tau,\mathrm{bdry}}(\theta,\phi)
		&:= \EE\bigl[(\mu(X)-h_{\theta,\phi,\tau}(X))^2
		\,\1_{\{X\in\mathcal{U}_\varepsilon(\tau;\phi)\}}\bigr].
	\end{align*}
	Moreover, there exists a constant $C(B,B_{\mathrm{stu}})$ such that
	\[
	0 \le \widetilde L_{\tau,\mathrm{bdry}}(\theta,\phi)
	\le C(B,B_{\mathrm{stu}})\,\mathrm{BM}^{\mathrm{amb}}_{\phi,\varepsilon}(\tau),
	\]
	for all $(\theta,\phi,\tau)$.
\end{lemma}

\begin{proof}
	We split the expectation over $\mathcal{G}_\varepsilon(\tau;\phi)$ and $\mathcal{U}_\varepsilon(\tau;\phi)$:
	\[
	\widetilde L_\tau(\theta,\phi)
	= \EE\bigl[(\mu(X)-h(X))^2\1_{\{X\in\mathcal{G}_\varepsilon(\tau;\phi)\}}\bigr]
	+ \EE\bigl[(\mu(X)-h(X))^2\1_{\{X\in\mathcal{U}_\varepsilon(\tau;\phi)\}}\bigr],
	\]
	where $h=h_{\theta,\phi,\tau}$. This gives the stated decomposition directly and avoids any convention about conditional expectations on events of probability zero.

	For the upper bound on the boundary term, on $\mathcal{U}_\varepsilon(\tau;\phi)$ we have
	$(\mu-h)^2\le 2\mu^2+2h^2$. The boundedness part of Assumption~\ref{ass:standing} gives
	$\abs{\mu(X)}\le B$ a.s. On the other hand,
	The identity
	\[
	h(X)=\sum_{k=1}^K p_k^{(\tau)}(X;\phi) f_k(X;\theta_k),
	\qquad
	\sum_{k=1}^K p_k^{(\tau)}(X;\phi)=1
	\]
	together with the uniform-boundedness hypothesis in Assumption~\ref{ass:student-moments}
	yields $\abs{h(X)}\le B_{\mathrm{stu}}$ a.s. Therefore
	\[
	(\mu(X)-h(X))^2\le 2B^2+2B_{\mathrm{stu}}^2
	\qquad\text{a.s.}
	\]
	and hence
	\[
	\widetilde L_{\tau,\mathrm{bdry}}(\theta,\phi)
	\le (2B^2+2B_{\mathrm{stu}}^2)\,\PP(X\in\mathcal{U}_\varepsilon(\tau;\phi))
	= (2B^2+2B_{\mathrm{stu}}^2)\,\mathrm{BM}^{\mathrm{amb}}_{\phi,\varepsilon}(\tau).
	\]
	This proves the claimed bound.
\end{proof}

\begin{remark}[Why the moment-only alternative is not used here]
		A global fourth-moment bound is enough to ensure that $\widetilde L_\tau(\theta,\phi)$ is finite, but it does not by itself yield a uniform bound on conditional second moments over the shrinking ambiguity regions $\mathcal U_\varepsilon(\tau;\phi)$. The linear boundary-layer estimate above therefore uses the bounded-output regime. A moment-only version would require an additional localized integrability hypothesis on those ambiguity sets.
	\end{remark}
	
	Thus the contribution of the boundary tubes to the risk is controlled by the
	boundary mass. For small~$\tau$, the next lemma gives
	$\mathrm{BM}^{\mathrm{amb}}_{\phi,\varepsilon}(\tau)=O(\tau)$ under mild regularity, so the boundary
	term is $O(\tau)$ as well.
	
	\begin{lemma}[Linear boundary-mass scaling under regularity]\label{lem:BM-linear}
		Suppose the fixed router parameter $\phi$ satisfies the pairwise tubular coarea regularity of Assumption~\ref{ass:router-coarea-regularity}. Then for each fixed
		$\varepsilon\in(0,1/2)$ there is $C_\varepsilon<\infty$ such that for all $\tau$
		sufficiently small,
		\[
		\mathrm{BM}^{\mathrm{amb}}_{\phi,\varepsilon}(\tau)
		\le C_\varepsilon \tau.
		\]
	\end{lemma}
	
	\begin{proof}
		By Remark~\ref{rem:BM-benign-vs-geom}, the $K$-way ambiguity event is contained in a finite union of pairwise slabs:
		\[
		\begin{aligned}
			\{X\in\mathcal{U}_\varepsilon(\tau;\phi)\}
			&\subseteq \{\Delta(X;\phi)\le \kappa_{\varepsilon,K}\tau\} \\
			&\subseteq \{\Delta_{\min}(X;\phi)\le \kappa_{\varepsilon,K}\tau\} \\
			&= \bigcup_{1\le k<\ell\le K}
			\bigl\{|a_k(X;\phi)-a_\ell(X;\phi)|\le \kappa_{\varepsilon,K}\tau\bigr\}.
		\end{aligned}
		\]
		Therefore,
		\[
		\mathrm{BM}^{\mathrm{amb}}_{\phi,\varepsilon}(\tau)
		\le \sum_{1\le k<\ell\le K}\PP\bigl(|S_{k\ell}(X;\phi)|\le \kappa_{\varepsilon,K}\tau\bigr).
		\]
		By the coarea formula and the level-set profile bound from Assumption~\ref{ass:router-coarea-regularity},
		\[
		\PP\bigl(|S_{k\ell}(X;\phi)|\le \kappa_{\varepsilon,K}\tau\bigr)
		= \int_{-\kappa_{\varepsilon,K}\tau}^{\kappa_{\varepsilon,K}\tau}
		\int_{\cX\cap\{S_{k\ell}(\cdot;\phi)=t\}} \frac{p_X(x)}{\|\nabla_x S_{k\ell}(x;\phi)\|}
		\,d\mathcal{H}^{d-1}(x)\,dt
		\le C_{k\ell}(\varepsilon)\,\tau,
		\]
		where $C_{k\ell}(\varepsilon)$ is finite by the level-set profile bound in Assumption~\ref{ass:router-coarea-regularity}.
		Summing over $k<\ell$ gives $\mathrm{BM}^{\mathrm{amb}}_{\phi,\varepsilon}(\tau)\le C_\varepsilon\tau$ for a finite constant $C_\varepsilon$.
	\end{proof}
	
	\subsection{Error control on correctly dominated interiors}
	
	Inside a teacher cell~$\Omega_k^\star$, the router ideally assigns almost all
	mass to expert~$k$. In this regime, the MoE behaves like a single-expert model,
	and a pointwise quadratic comparison controls the error of
	$f_k(\cdot;\theta_k)$ on the correctly dominated interior.

	\begin{proposition}[Cellwise error control on the correctly dominated interior]\label{prop:cell-error}
		Fix $k$ and define
		\[
		A_{k,\varepsilon}(\theta,\phi,\tau)
		:= \Omega_k^\star\cap\mathcal{G}_\varepsilon(\tau;\phi)\cap\{p_k^{(\tau)}(X;\phi)\ge 1-\varepsilon\}.
		\]
		Under Assumption~\ref{ass:standing} and the uniform-boundedness branch of Assumption~\ref{ass:student-moments},
		\[
		\begin{aligned}
			&\EE\Bigl[
			\bigl(\mu(X)-h_{\theta,\phi,\tau}(X)\bigr)^2\,
			\1_{A_{k,\varepsilon}(\theta,\phi,\tau)}
			\Bigr] \\
			&\qquad\ge \frac12\,\EE\Bigl[
			\abs{f_k(X;\theta_k)-f_k^\star(X)}^2\,
			\1_{A_{k,\varepsilon}(\theta,\phi,\tau)}
			\Bigr]
			-4B_{\mathrm{stu}}^2\varepsilon^2\,
			P_X\bigl(A_{k,\varepsilon}(\theta,\phi,\tau)\bigr).
		\end{aligned}
		\]
	\end{proposition}
	
	\begin{proof}
		Fix $k$ and restrict attention to $X\in A_{k,\varepsilon}(\theta,\phi,\tau)$.
		On this region the teacher output is $f_k^\star(X)$ by definition of the
		teacher model~\eqref{eq:teacher}, and the router gives expert $k$ weight at least $1-\varepsilon$.
		
		Write
		\[
		h_{\theta,\phi,\tau}(X)
		= p_k^{(\tau)}(X;\phi)\,f_k(X;\theta_k)
		+ \sum_{\ell\ne k}p_\ell^{(\tau)}(X;\phi)\,f_\ell(X;\theta_\ell).
		\]
		Then
		\[
		f_k^\star(X)-h_{\theta,\phi,\tau}(X)
		= \bigl(f_k^\star(X)-f_k(X;\theta_k)\bigr)
		+ \Bigl( (1-p_k^{(\tau)}(X;\phi))f_k(X;\theta_k)
		- \sum_{\ell\ne k}p_\ell^{(\tau)}(X;\phi)\,f_\ell(X;\theta_\ell)\Bigr).
		\]
		On $A_{k,\varepsilon}(\theta,\phi,\tau)$ we have $1-p_k^{(\tau)}(X;\phi)\le \varepsilon$ and $\sum_{\ell\ne k}p_\ell^{(\tau)}(X;\phi)\le \varepsilon$ pointwise.
		Using $(a+b)^2\ge \tfrac12 a^2-b^2$ with
		$a=f_k^\star(X)-f_k(X;\theta_k)$ and $b$ equal to the remaining parenthesis, we get
		\[
		(f_k^\star(X)-h_{\theta,\phi,\tau}(X))^2
		\ge \tfrac12\,\abs{f_k^\star(X)-f_k(X;\theta_k)}^2 - \abs{b}^2.
		\]
		The boundedness assumption gives
		\[
		\abs{b}
		\le (1-p_k^{(\tau)}(X;\phi))\abs{f_k(X;\theta_k)} + \sum_{\ell\ne k}p_\ell^{(\tau)}(X;\phi)\abs{f_\ell(X;\theta_\ell)}
		\le 2\varepsilon B_{\mathrm{stu}}.
		\]
		Hence
		\[
		(\mu(X)-h_{\theta,\phi,\tau}(X))^2\1_{A_{k,\varepsilon}}
		\ge \frac12\abs{f_k^\star(X)-f_k(X;\theta_k)}^2\1_{A_{k,\varepsilon}}-4B_{\mathrm{stu}}^2\varepsilon^2\1_{A_{k,\varepsilon}}.
		\]
		Taking expectations gives the claim.
	\end{proof}
	
	\begin{remark}[Moment-only variants]
		The unconditional form above is chosen so that no conditional $L^2$ estimate on $A_{k,\varepsilon}$ is needed. A version based only on fourth moments would again require extra localized integrability on the events where the router is nearly deterministic.
	\end{remark}

	\subsection{Shape derivatives and strict saddles}
	
	To connect boundary motion with optimization geometry, we next use the standard shape-derivative viewpoint \citep{HenrotPierre2005,SokolowskiZolesio1992}. The idea is concrete. If a router boundary is moved across a region where one expert gives a lower squared error than another, the risk changes to first order by the loss contrast across that interface. When the first variation vanishes but the second variation is negative along a realized boundary motion, this produces a strict-saddle direction.
	
	Consider a fixed pair of neighboring cells $(\Omega_k,\Omega_\ell)$ separated
	by a smooth interface $\cS_{k\ell}$ with outward normal $n_{k\ell}$ pointing
	from $\Omega_k$ into~$\Omega_\ell$. A \emph{normal deformation} of the
	interface is specified by a smooth scalar field
	$\varphi:\cS_{k\ell}\to\RR$. Points on $\cS_{k\ell}$ are pushed along the
	normal by $\epsilon\varphi(x)n_{k\ell}(x)$ for small~$\epsilon$. This induces a
	family of perturbed domains $(\Omega_k^\epsilon,\Omega_\ell^\epsilon)$ and
	hence a perturbed partition $(\Omega_j^\epsilon)_{j=1}^K$.
	
	Define the \emph{loss contrast}
	\begin{equation}\label{eq:loss-contrast-Dkl}
		D_{k\ell}(x;\theta,\mu)
		:= \bigl(f_k(x;\theta_k)-\mu(x)\bigr)^2
		- \bigl(f_\ell(x;\theta_\ell)-\mu(x)\bigr)^2.
	\end{equation}
	This notation is kept separate from the routing margin $\Delta(x;\phi)$ in~\eqref{eq:routing-margin}. The quantity $D_{k\ell}(x;\theta,\mu)$ measures how much worse expert~$k$ is compared to expert~$\ell$ at point~$x$ when the target is $\mu(x)$.

	\begin{proposition}[Localized shape derivative of risk]\label{prop:shape-derivative}
		Suppose Assumption~\ref{ass:standing} holds, and fix an adjacent pair $(k,\ell)$ for the induced hard partition.
		Let $\cP\Subset \cS_{k\ell}$ be a relatively compact $C^2$ patch of the interface whose closure stays away from all multiway junctions.
		Assume that $p_X$ has a continuous trace on $\cP$ and that the teacher signal $\mu$ admits one-sided traces on $\cP$ from $\Omega_k$ and $\Omega_\ell$; this holds, for example, when the corresponding teacher experts extend continuously to the patch from their respective sides.
		Consider a normal deformation of $\cP$ given by a smooth compactly supported field
		$\varphi\in C_c^\infty(\cP)$, and write $\varphi_+:=\max\{\varphi,0\}$ and $\varphi_-:=\max\{-\varphi,0\}$.
		Let $\mu_k$ and $\mu_\ell$ denote the one-sided traces of $\mu$ on $\cP$ from $\Omega_k$ and $\Omega_\ell$, respectively. Then the first variation of the hard-routing risk
		\[
		\widetilde L_0(\theta,\phi)
		:= \EE\bigl[(\mu(X)-h_{\theta,\phi,0}(X))^2\bigr]
		\]
		with respect to this localized two-phase deformation is
		\begin{equation}\label{eq:shape-derivative-localized}
			\frac{\mathrm{d}}{\mathrm{d}\epsilon}
			\widetilde L_0(\theta,\phi_\epsilon)\biggr|_{\epsilon=0}
			= \int_{\cP}\Bigl[
			\varphi_+(x)\,D_{k\ell}(x;\theta,\mu_\ell)
			-\varphi_-(x)\,D_{k\ell}(x;\theta,\mu_k)
			\Bigr]p_X(x)
			\,\mathrm{d}\mathcal{H}^{d-1}(x).
		\end{equation}
		In particular, if $\varphi\ge 0$ on $\cP$, this reduces to
		\[
		\frac{\mathrm{d}}{\mathrm{d}\epsilon}
		\widetilde L_0(\theta,\phi_\epsilon)\biggr|_{\epsilon=0}
		= \int_{\cP}
		\varphi(x)\,D_{k\ell}(x;\theta,\mu_\ell)\,p_X(x)
		\,\mathrm{d}\mathcal{H}^{d-1}(x),
		\]
		and if $\varphi\le 0$ on $\cP$, it becomes
		\[
		\frac{\mathrm{d}}{\mathrm{d}\epsilon}
		\widetilde L_0(\theta,\phi_\epsilon)\biggr|_{\epsilon=0}
		= \int_{\cP}
		\varphi(x)\,D_{k\ell}(x;\theta,\mu_k)\,p_X(x)
		\,\mathrm{d}\mathcal{H}^{d-1}(x).
		\]
		Moreover, if the first variation vanishes for a nonzero realized velocity and the corresponding localized second shape derivative is strictly negative, then the hard-routing objective has a negative-curvature direction along that localized two-phase deformation.
	\end{proposition}
	
	\begin{proof}
		Because the deformation is supported on a smooth patch $\cP$ away from multiway junctions, the problem is locally two-phase and the standard Hadamard--Zol\'esio transport formula applies; see, for example, \citep[Chs.~2--3]{SokolowskiZolesio1992} or \citep[Ch.~5]{HenrotPierre2005}. Write the hard risk locally as
		\[
		J(\epsilon)=\int_{\Omega_k^\epsilon} F_k(x)\,p_X(x)\,dx + \int_{\Omega_\ell^\epsilon} F_\ell(x)\,p_X(x)\,dx,
		\qquad
		F_j(x):=(\mu(x)-f_j(x;\theta_j))^2,
		\]
		with all other cells unchanged. When $\varphi(x)>0$, the interface moves in the normal direction from $\Omega_k$ into $\Omega_\ell$, so $\Omega_k$ gains an infinitesimal layer that originally belonged to $\Omega_\ell$. The transported target on that layer is therefore the $\Omega_\ell$-trace $\mu_\ell$, and the jump contribution is $D_{k\ell}(x;\theta,\mu_\ell)$. Likewise, when $\varphi(x)<0$, the phase $\Omega_k$ loses an infinitesimal layer to $\Omega_\ell$, so the relevant trace is $\mu_k$ and the contribution is $D_{k\ell}(x;\theta,\mu_k)$ multiplied by the negative velocity. Splitting $\varphi=\varphi_+-\varphi_-$ therefore gives \eqref{eq:shape-derivative-localized}.
		
		The final claim is the usual second-order consequence of the shape calculation. If the first variation vanishes and the second shape derivative is strictly negative along a realized deformation, then the hard objective decreases to second order along that deformation, so the corresponding quadratic form is negative.
	\end{proof}

	Thus, on any smooth two-phase patch away from junctions, a vanishing first variation together with strictly negative second variation yields a genuine local shape descent direction. To transfer this to the parameter variable $\phi$, we will only need a realizability condition ensuring that the router can produce the finite-dimensional family of interface velocities that actually arise from its own parameterization.
	
	\subsection{Global minima and strict saddles under the transfer assumptions}
	
	We now assemble the assumptions needed for the transfer statement. A pointwise contrast condition along teacher interfaces would be stronger than necessary. What the proof actually uses is a probabilistic separation condition: after the correct relabeling, a wrong expert should not approximate a teacher expert on too much probability mass. The condition is stated modulo relabeling, as is standard for MoE models, because the expert indices themselves carry no intrinsic meaning.
	
	We say that an induced hard partition $(\Omega_k(\phi))_{k=1}^K$ is \emph{teacher-equivalent up to permutation} if there exists a permutation $\pi$ of $\{1,\dots,K\}$ such that $P_X(\Omega_{\pi(k)}(\phi)\triangle\Omega_k^\star)=0$ for every teacher cell $k$.
	
	\begin{assumption}[Probabilistic expert-separation condition modulo relabeling]\label{ass:expert-gap}
		There exist constants $\gamma_{\mathrm{gap}}>0$ and $\eta_{\mathrm{gap}}\in[0,1)$ such that for every student parameter vector $\theta$ in the region under consideration there is at least one permutation $\pi=\pi(\theta)$ of $\{1,\dots,K\}$ satisfying
		\begin{equation}\label{eq:expert-gap-matching}
		P_X\Bigl(\bigl\{x\in\Omega_k^\star:\ (f_j(x;\theta_j)-f_k^\star(x))^2 \le \gamma_{\mathrm{gap}}\bigr\}\Bigr)
		\le \eta_{\mathrm{gap}}
		\end{equation}
		for every teacher cell $k$ and every student expert $j\neq\pi(k)$. Any permutation satisfying~\eqref{eq:expert-gap-matching} is called an admissible expert matching for~$\theta$.
	\end{assumption}
	Assumption~\ref{ass:expert-gap} says that, after one admissible one-to-one matching between student experts and teacher cells has been chosen, the unmatched experts cannot approximate the matched teacher expert on a set of non-negligible probability mass. In the limiting case $\eta_{\mathrm{gap}}=0$, no unmatched expert matches a given teacher expert on a set of positive measure inside that teacher cell.

	We also require router realizability of normal deformations. To connect shape-derivative directions to actual parameter directions, we use the following notation and hypothesis.
	
	\begin{assumption}[Router regularity and realized interface directions]\label{ass:router-shape-surjectivity}
		The router logits $a_k(\cdot;\phi)$ are $C^2$ in $x$ and $C^2$ in $\phi$. For each operating point $\phi$, each adjacent pair $(k,\ell)$, and each smooth patch $\cP\Subset \cS_{k\ell}(\phi)$ away from multiway junctions, the differential of the router defines a finite-dimensional subspace $\mathcal V_{k\ell}(\phi,\cP)\subset C_c^\infty(\cP)$ of realizable normal velocities. More precisely, every parameter direction $\zeta$ induces a first-order normal velocity $V_{k\ell}[\zeta]\cdot n_{k\ell}\in \mathcal V_{k\ell}(\phi,\cP)$, and conversely every $\varphi\in \mathcal V_{k\ell}(\phi,\cP)$ is produced by at least one such parameter direction.
	\end{assumption}
	
	Assumption~\ref{ass:router-shape-surjectivity} is a finite-dimensional realizability condition. We do not ask the router to produce every smooth boundary motion. We only ask that the nondegeneracy argument be carried out inside the family of interface velocities that the parameterization can actually see. For linear logits, for instance, the induced normal velocities on a fixed smooth patch are finite-dimensional affine functions of the ambient features, not arbitrary elements of $C_c^\infty(\cP)$.
	
	The assumptions enter the theorem in a specific order. Boundary-mass control and teacher-router realizability provide the reference partition. The probabilistic separation condition prevents unmatched experts from approximating the wrong teacher cell on too much probability mass. Router realizability identifies the interface velocities represented by the parameterization. The final two assumptions supply the genuinely conditional inputs: first- and second-order transfer from the hard objective to the soft objective, and negative-curvature geometry of the profiled hard problem away from the teacher partition. With these ingredients, the theorem is a reduction statement rather than a general description of all MoE objectives.
	
	\begin{assumption}[Soft-to-hard $C^1/C^2$ control near critical points]\label{ass:soft-hard-C2}
		Fix $\varepsilon\in(0,1/2)$. Work in finite-dimensional local charts in which $L_\tau$ and $\widetilde L_0$ are twice continuously differentiable near the critical points under consideration. There exist $\tau_0>0$ and a function $r(\tau)\downarrow 0$ as $\tau\downarrow 0$ such that for all $\tau\in(0,\tau_0]$ and all $(\theta,\phi)$ in those neighborhoods,
		\[
		\|\nabla_\phi L_\tau(\theta,\phi)-\nabla_\phi\widetilde L_0(\theta,\phi)\|+\|\nabla_\theta L_\tau(\theta,\phi)-\nabla_\theta\widetilde L_0(\theta,\phi)\|\le r(\tau),
		\]
		and the Hessian blocks satisfy
		\[
		\|\nabla^2_{\phi\phi}L_\tau-\nabla^2_{\phi\phi}\widetilde L_0\|_{\mathrm{op}}+\|\nabla^2_{\theta\phi}L_\tau-\nabla^2_{\theta\phi}\widetilde L_0\|_{\mathrm{op}}+\|\nabla^2_{\theta\theta}L_\tau-\nabla^2_{\theta\theta}\widetilde L_0\|_{\mathrm{op}}\le r(\tau).
		\]
	\end{assumption}
	
	Assumption~\ref{ass:soft-hard-C2} replaces the earlier Hessian-only hypothesis. The strict-saddle step needs gradient transfer as well, because soft criticality must imply approximate hard criticality before one can invoke the shape calculus.
	
	\begin{assumption}[Fixed-partition expert geometry and realized interface nondegeneracy]\label{ass:frozen-profiled-landscape}
		There exists $\alpha_{\mathrm{fr}}>0$ such that for every router parameter $\phi$ in the region under consideration, the hard-routing risk $\theta\mapsto\widetilde L_0(\theta,\phi)$ has a unique critical point $\hat\theta(\phi)$, and this critical point is the global minimizer in the expert variables. In the local expert chart containing the critical points under consideration,
		\[
		\nabla^2_{\theta\theta}\widetilde L_0(\theta,\phi)\succeq \alpha_{\mathrm{fr}}I.
		\]
		The map $\phi\mapsto\hat\theta(\phi)$ is $C^1$ in the local chart, with uniformly bounded derivative. Finally, whenever the induced hard partition at $\phi$ is not teacher-equivalent up to permutation, there exist an adjacent pair $(k,\ell)$, a smooth patch $\cP\Subset \cS_{k\ell}(\phi)$ away from multiway junctions, a nonzero realized velocity $\varphi\in\mathcal V_{k\ell}(\phi,\cP)$, and a router-parameter direction $\zeta$ realizing that velocity such that the first variation of the profiled hard objective
		\[
		\Psi_0(\phi):=\widetilde L_0(\hat\theta(\phi),\phi)
		\]
		vanishes along $\zeta$ and the corresponding full tangent direction
		\[
		q_\zeta:=\bigl(D\hat\theta(\phi)[\zeta],\zeta\bigr)
		\]
		has strictly negative hard quadratic form:
		\[
		q_\zeta^\top \nabla^2\widetilde L_0(\hat\theta(\phi),\phi)q_\zeta\le -\lambda_{\mathrm{nd}}\|q_\zeta\|^2.
		\]
		The constant $\lambda_{\mathrm{nd}}>0$ is independent of the critical point.
	\end{assumption}
	
	\begin{lemma}[Negative curvature persists under small $C^2$ perturbations]\label{lem:weyl-perturb}
		Let $H_0$ and $H$ be symmetric matrices (or bounded self-adjoint operators) with
		$\|H-H_0\|_{\mathrm{op}}\le\eta$. Then $\lambda_{\min}(H)\le \lambda_{\min}(H_0)+\eta$.
		In particular, if $\lambda_{\min}(H_0)\le -\lambda$ and $\eta\le \lambda/2$, then $\lambda_{\min}(H)\le -\lambda/2$.
	\end{lemma}
	\begin{proof}
		This is a standard consequence of the variational characterization
		$\lambda_{\min}(H)=\inf_{\|u\|=1} u^\top H u$ and the inequality
		$|u^\top(H-H_0)u|\le\|H-H_0\|_{\mathrm{op}}$.
	\end{proof}
	
	For a router parameter $\phi$, we denote its hard cells by
	\begin{equation}\label{eq:induced-hard-cells}
		\Omega_k(\phi):=\{x\in\cX:\ k=k^\star(x;\phi)\},
	\end{equation}
	where $k^\star(x;\phi)$ is the unique maximizing index outside the tie set and ties are assigned by a fixed deterministic rule on the remaining null set. For a candidate relabeling $\pi$ and parameters $(\theta,\phi)$, define the confidently misrouted interior set
	\begin{equation}\label{eq:misroute-set-pi}
		\mathcal E_{\mathrm{mis}}(\pi;\theta,\phi,\tau):=\Bigl\{x\in\cX:\ x\in\Omega_k^\star,
		p_j^{(\tau)}(x;\phi)\ge 1-\varepsilon\text{ for some }j\neq \pi(k)\Bigr\}\cap\mathcal G_\varepsilon(\tau;\phi).
	\end{equation}
	When $(\theta,\phi,\tau)$ are fixed, we write this set simply as $\mathcal E_{\mathrm{mis}}(\pi)$.
	
	\begin{lemma}[Misrouting mass]\label{lem:misrouting-matching}
		Assume Assumption~\ref{ass:standing}, the uniform-boundedness branch of Assumption~\ref{ass:student-moments}, and Assumption~\ref{ass:expert-gap}. Fix $\varepsilon\in(0,1/2)$ such that $2\varepsilon B_{\mathrm{stu}}\le\frac12\sqrt{\gamma_{\mathrm{gap}}}$. Then the following bound holds for every parameter pair $(\theta,\phi)$, every temperature $\tau>0$, and every admissible expert matching $\pi$ for~$\theta$:
		\begin{equation}\label{eq:misroute-lint-lower}
			\widetilde L_{\tau,\mathrm{int}}(\theta,\phi)
			\ge \frac{\gamma_{\mathrm{gap}}}{4}\,\max\{P_X(\mathcal E_{\mathrm{mis}}(\pi;\theta,\phi,\tau))-K(K-1)\eta_{\mathrm{gap}},0\}.
		\end{equation}
		Consequently, if $\widetilde L_\tau(\theta,\phi)\le R$, then
		\begin{equation}\label{eq:misroute-mass-consequence}
			P_X(\mathcal E_{\mathrm{mis}}(\pi;\theta,\phi,\tau))
			\le \frac{4R}{\gamma_{\mathrm{gap}}}+K(K-1)\eta_{\mathrm{gap}}.
		\end{equation}
	\end{lemma}
	
	\begin{proof}
		Fix $\pi$. On $\mathcal E_{\mathrm{mis}}(\pi;\theta,\phi,\tau)$, the point belongs to some teacher cell $\Omega_k^\star$ and the soft router assigns weight at least $1-\varepsilon$ to a student expert $j\neq\pi(k)$. The bounded-expert hypothesis gives
		\begin{equation}\label{eq:wrong-expert-dominance}
			\abs{h_{\theta,\phi,\tau}(x)-f_j(x;\theta_j)}
			\le \sum_{m\neq j}p_m^{(\tau)}(x;\phi)
			\abs{f_m(x;\theta_m)-f_j(x;\theta_j)}
			\le 2\varepsilon B_{\mathrm{stu}}.
		\end{equation}
		For each such pair $k$ and $j\neq\pi(k)$, introduce
		\begin{equation}\label{eq:gap-exceptional-set}
			E_{kj}:=\Bigl\{x\in\Omega_k^\star:(f_j(x;\theta_j)-f_k^\star(x))^2\le \gamma_{\mathrm{gap}}\Bigr\},
			\qquad
			E_{\mathrm{gap}}(\pi):=\bigcup_{k=1}^K\bigcup_{j\neq\pi(k)}E_{kj}.
		\end{equation}
		Because $\pi$ is an admissible expert matching, Assumption~\ref{ass:expert-gap} yields $P_X(E_{\mathrm{gap}}(\pi))\le K(K-1)\eta_{\mathrm{gap}}$. If $x\in\mathcal E_{\mathrm{mis}}(\pi;\theta,\phi,\tau)\setminus E_{\mathrm{gap}}(\pi)$, then the dominating expert is at least $\sqrt{\gamma_{\mathrm{gap}}}$ away from the teacher signal on that cell before the soft-mixture error in~\eqref{eq:wrong-expert-dominance}. Since $2\varepsilon B_{\mathrm{stu}}\le\frac12\sqrt{\gamma_{\mathrm{gap}}}$, the squared residual is at least $\gamma_{\mathrm{gap}}/4$ on this set. Integrating over the correctly restricted interior region gives~\eqref{eq:misroute-lint-lower}, and~\eqref{eq:misroute-mass-consequence} follows because $\widetilde L_{\tau,\mathrm{int}}\le\widetilde L_\tau$.
	\end{proof}
	
	The following theorem is the main conditional optimization statement of the paper. Its two conclusions have different sources. The global-minimizer estimate follows from boundary-layer control, comparison with the teacher configuration, and the matching lemma. The strict-saddle conclusion uses the stronger conditional inputs: negative curvature for the profiled hard-routing problem and first- and second-order transfer from the hard objective to the soft objective.
	
	\begin{theorem}[Conditional landscape-transfer principle]\label{thm:benign}
		Suppose all assumptions stated above in this section hold, and assume the uniform-boundedness branch of Assumption~\ref{ass:student-moments}. Fix $\varepsilon\in(0,1/2)$ so that
		\[
		2\varepsilon B_{\mathrm{stu}}\le \frac12\sqrt{\gamma_{\mathrm{gap}}}.
		\]
		Then there exists $\tau_0>0$ such that the following conclusions hold for every $\tau\in(0,\tau_0]$. For any global minimizer $(\theta,\phi)$ of~$L_\tau$, there is a permutation $\pi$ of $\{1,\dots,K\}$ such that
		\begin{equation}\label{eq:benign-identifiability}
			\sum_{k=1}^K
			\EE\bigl[
			\abs{f_{\pi(k)}(X;\theta_{\pi(k)})-f_k^\star(X)}^2
			\,\1_{\Omega_k^\star}(X)
			\bigr]
			\le C_1\,\Bigl(\mathrm{BM}^{\mathrm{amb}}_{\phi^\star,\varepsilon}(\tau)+\mathrm{BM}^{\mathrm{amb}}_{\phi,\varepsilon}(\tau)+\eta_{\mathrm{gap}}+\varepsilon^2\Bigr),
		\end{equation}
		and the induced partition $(\Omega_k(\phi))$ satisfies
		\begin{equation}\label{eq:benign-partition}
			\sum_{k=1}^K
			\PP\bigl( (\Omega_{\pi(k)}(\phi)\triangle\Omega_k^\star)\cap\mathcal{G}_\varepsilon(\tau;\phi)
			\bigr)
			\le C_2\,\bigl(\mathrm{BM}^{\mathrm{amb}}_{\phi^\star,\varepsilon}(\tau)+\eta_{\mathrm{gap}}+\varepsilon^2\bigr),
		\end{equation}
		where $C_1$ and $C_2$ are independent of $\tau$ for fixed $\varepsilon$ and $\eta_{\mathrm{gap}}$. In addition, every critical point $(\theta,\phi)$ of~$L_\tau$ whose induced hard partition is not teacher-equivalent up to permutation is a strict saddle, in the sense that the Hessian of~$L_\tau$ at $(\theta,\phi)$ has at least one negative eigenvalue.
	\end{theorem}
	
	\begin{proof}
		All constants below are uniform for $\tau\in(0,\tau_0]$. We first isolate the boundary-layer contribution and then compare the soft objective with the teacher configuration away from that layer. This yields the global-minimum statement. We then turn to the remaining critical points and show that the hard negative curvature persists for sufficiently small temperature.
		
		The boundary contribution is the first term to control. By Lemma~\ref{lem:risk-decomp},
		\[
		\widetilde L_\tau(\theta,\phi)
		= \widetilde L_{\tau,\mathrm{int}}(\theta,\phi)
		+ \widetilde L_{\tau,\mathrm{bdry}}(\theta,\phi),
		\qquad
		\widetilde L_{\tau,\mathrm{bdry}}(\theta,\phi)
		\le C\,\mathrm{BM}^{\mathrm{amb}}_{\phi,\varepsilon}(\tau).
		\]
		Thus every place where the soft model can differ substantially from the hard-routing picture is already confined to a boundary layer of size $O(\mathrm{BM}^{\mathrm{amb}}_{\phi,\varepsilon}(\tau))$.
		
		We then compare with the teacher configuration and quantify the cost of confident misrouting. At the teacher parameters $(\theta^\star,\phi^\star)$, the hard predictor agrees with the teacher signal outside a $P_X$-null tie set. For the softened teacher router, a direct bounded-output estimate is enough. On $\mathcal G_\varepsilon(\tau;\phi^\star)$, the correct teacher expert receives weight at least $1-\varepsilon$ because the teacher router realizes the teacher partition and the winning weight is the largest softmax weight. Hence the soft prediction differs from $\mu$ by at most $2B_{\mathrm{stu}}\varepsilon$ on this region. On $\mathcal U_\varepsilon(\tau;\phi^\star)$ the squared discrepancy is bounded by a constant depending only on $B$ and $B_{\mathrm{stu}}$. Therefore
		\begin{equation}\label{eq:teacher-risk-upper}
			\widetilde L_\tau(\theta^\star,\phi^\star)
			\le C_0\,\mathrm{BM}^{\mathrm{amb}}_{\phi^\star,\varepsilon}(\tau)+C_0\varepsilon^2.
		\end{equation}
		We keep the $\varepsilon^2$ term explicit in the final estimates, because $\varepsilon$ is fixed before the small-temperature limit is taken.
		
		These estimates force global minimizers to stay close to the teacher after an admissible matching of student experts to teacher cells. Let $(\theta,\phi)$ be a global minimizer and choose an admissible expert matching $\pi$ for $\theta$ as provided by Assumption~\ref{ass:expert-gap}. Since global minimality gives $\widetilde L_\tau(\theta,\phi)\le\widetilde L_\tau(\theta^\star,\phi^\star)$, Lemma~\ref{lem:misrouting-matching} and \eqref{eq:teacher-risk-upper} imply
		\begin{equation}\label{eq:M-small-global-min}
			P_X(\mathcal E_{\mathrm{mis}}(\pi;\theta,\phi,\tau))
			\lesssim \mathrm{BM}^{\mathrm{amb}}_{\phi^\star,\varepsilon}(\tau)+\eta_{\mathrm{gap}}+\varepsilon^2.
		\end{equation}
		In this form the matching step is explicit. A one-to-one relabeling is chosen first, and the separation condition then rules out large confidently misrouted interior mass under that relabeling, except on the allowed exceptional sets and the boundary layer.
		
		On the good interior region
		\begin{equation}\label{eq:good-interior-region-Ak}
			A_k:=\Omega_k^\star\cap\mathcal G_\varepsilon(\tau;\phi)\cap \mathcal E_{\mathrm{mis}}(\pi;\theta,\phi,\tau)^c\cap\{p_{\pi(k)}^{(\tau)}(X;\phi)\ge1-\varepsilon\},
		\end{equation}
		the router gives expert $\pi(k)$ weight at least $1-\varepsilon$. Applying Proposition~\ref{prop:cell-error} after relabeling the student experts by $\pi$ and summing over $k$ gives
		\[
		\sum_{k=1}^K \EE\Bigl[\abs{f_{\pi(k)}(X;\theta_{\pi(k)})-f_k^\star(X)}^2\,\1_{A_k}(X)\Bigr]
		\lesssim \widetilde L_\tau(\theta,\phi)+\varepsilon^2.
		\]
		Because $(\theta,\phi)$ is globally minimizing,
		\[
		\widetilde L_\tau(\theta,\phi)
		\le \widetilde L_\tau(\theta^\star,\phi^\star)
		\lesssim \mathrm{BM}^{\mathrm{amb}}_{\phi^\star,\varepsilon}(\tau)+\varepsilon^2
		\]
		by \eqref{eq:teacher-risk-upper}. Moreover,
		\[
		\Omega_k^\star\setminus A_k\subseteq \mathcal U_\varepsilon(\tau;\phi)\cup\mathcal E_{\mathrm{mis}}(\pi;\theta,\phi,\tau).
		\]
		Since $\abs{f_{\pi(k)}(X;\theta_{\pi(k)})-f_k^\star(X)}^2\le 2B_{\mathrm{stu}}^2+2B^2$ almost surely, the contribution of $\Omega_k^\star\setminus A_k$ is bounded by a constant times $P_X(\mathcal U_\varepsilon(\tau;\phi)\cup\mathcal E_{\mathrm{mis}}(\pi;\theta,\phi,\tau))$. Summing over $k$ and using \eqref{eq:M-small-global-min} yields
		\[
		\sum_{k=1}^K
		\EE\bigl[
		\abs{f_{\pi(k)}(X;\theta_{\pi(k)})-f_k^\star(X)}^2\,\1_{\Omega_k^\star}(X)
		\bigr]
		\lesssim \mathrm{BM}^{\mathrm{amb}}_{\phi^\star,\varepsilon}(\tau)+\mathrm{BM}^{\mathrm{amb}}_{\phi,\varepsilon}(\tau)+\eta_{\mathrm{gap}}+\varepsilon^2,
		\]
		which is \eqref{eq:benign-identifiability} for the chosen labeling. Restoring the original labels gives the claimed permutation $\pi$.
		
		For the partition estimate, fix $x\in \mathcal G_\varepsilon(\tau;\phi)$. If $x\in\Omega_k^\star$ and the induced hard cell after relabeling is not $\Omega_{\pi(k)}(\phi)$, then the expert receiving weight at least $1-\varepsilon$ is wrong relative to the selected matching. Hence $x\in\mathcal E_{\mathrm{mis}}(\pi;\theta,\phi,\tau)$, apart from the fixed null tie set used in~\eqref{eq:induced-hard-cells}. Therefore
		\[
		\bigcup_{k=1}^K \bigl((\Omega_{\pi(k)}(\phi)\triangle\Omega_k^\star)\cap \mathcal G_\varepsilon(\tau;\phi)\bigr)
		\subseteq \mathcal E_{\mathrm{mis}}(\pi;\theta,\phi,\tau),
		\]
		and hence, after absorbing the combinatorial counting factor into the constant,
		\[
		\sum_{k=1}^K \PP\bigl((\Omega_{\pi(k)}(\phi)\triangle\Omega_k^\star)\cap \mathcal G_\varepsilon(\tau;\phi)\bigr)
		\lesssim P_X(\mathcal E_{\mathrm{mis}}(\pi;\theta,\phi,\tau))
		\lesssim \mathrm{BM}^{\mathrm{amb}}_{\phi^\star,\varepsilon}(\tau)+\eta_{\mathrm{gap}}+\varepsilon^2,
		\]
		which is \eqref{eq:benign-partition}.
		
		It remains to analyze critical points whose induced hard partition is not teacher-equivalent up to permutation. Let $(\theta,\phi)$ be such a critical point of $L_\tau$. Since $\nabla L_\tau(\theta,\phi)=0$ and Assumption~\ref{ass:soft-hard-C2} gives gradient control, the hard objective is approximately critical at the same point:
		\[
		\|\nabla_\theta\widetilde L_0(\theta,\phi)\|+\|\nabla_\phi\widetilde L_0(\theta,\phi)\|\le r(\tau).
		\]
		For the fixed hard partition induced by $\phi$, Assumption~\ref{ass:frozen-profiled-landscape} gives a unique expert minimizer $\hat\theta(\phi)$ and a uniformly positive $\theta\theta$ Hessian. The strong-convexity estimate in the expert block gives
		\[
		\|\theta-\hat\theta(\phi)\|\le \alpha_{\mathrm{fr}}^{-1}r(\tau)
		\]
		inside the local chart. Since the induced hard partition is not teacher-equivalent, Assumption~\ref{ass:frozen-profiled-landscape} supplies a router direction $\zeta$ and the corresponding full tangent direction
		\[
		q_\zeta=\bigl(D\hat\theta(\phi)[\zeta],\zeta\bigr)
		\]
		for which
		\[
		q_\zeta^\top \nabla^2\widetilde L_0(\hat\theta(\phi),\phi)q_\zeta\le -\lambda_{\mathrm{nd}}\|q_\zeta\|^2.
		\]
		Because $\widetilde L_0$ is $C^2$ in the local chart and $\theta$ is $O(r(\tau))$ from $\hat\theta(\phi)$, decreasing $\tau_0$ if necessary gives
		\[
		q_\zeta^\top \nabla^2\widetilde L_0(\theta,\phi)q_\zeta\le -\frac{3}{4}\lambda_{\mathrm{nd}}\|q_\zeta\|^2.
		\]
		Assumption~\ref{ass:soft-hard-C2} gives
		\[
		\|\nabla^2L_\tau(\theta,\phi)-\nabla^2\widetilde L_0(\theta,\phi)\|_{\mathrm{op}}\le c\,r(\tau)
		\]
		for a chart-dependent constant $c$ coming only from the block norm convention. Choosing $\tau_0$ still smaller so that $c r(\tau)\le \lambda_{\mathrm{nd}}/4$ yields
		\[
		q_\zeta^\top \nabla^2L_\tau(\theta,\phi)q_\zeta\le -\frac12\lambda_{\mathrm{nd}}\|q_\zeta\|^2<0.
		\]
		Thus the Hessian of $L_\tau$ has a negative quadratic form, hence a negative eigenvalue by Lemma~\ref{lem:weyl-perturb}. The critical point is a strict saddle.
		
		This proves the theorem.
	\end{proof}
	
	\begin{remark}[Landscape and optimization]
		Theorem~\ref{thm:benign} is a population transfer statement. Its strict-saddle conclusion relies on two hypotheses that are not derived in this paper: negative-curvature geometry for the profiled hard-routing objective and derivative transfer from soft to hard routing. Under those hypotheses, critical points with non-teacher-equivalent hard partitions have a negative-curvature direction for sufficiently small $\tau$. This is consistent with the standard strict-saddle optimization picture: \citet{Ge2015} prove an escape result for stochastic gradient methods in a strict-saddle setting, while \citet{Sun2017Benign} give a representative benign-landscape analysis in phase retrieval. A rigorous SGD theorem for the present MoE objective would require a separate dynamical analysis.
	\end{remark}


	\section{A reduced Gaussian calculation for symmetry breaking}\label{sec:symmetry-breaking}
	
	The transfer theorem assumes favorable structure in the hard-routing limit; it does not explain how such structure might arise during training. This section studies a narrower local mechanism. Consider a two-expert Gaussian model in which the experts have already developed a small contrast while the router remains close to balanced. A population linearization then identifies whether the router has an unstable direction that points toward the teacher separator. The calculation is intentionally reduced. It is meant to isolate a symmetry-breaking mechanism, not to prove convergence of end-to-end MoE training.
	
	\subsection{Setup}\label{subsec:symmetry-setup}
	
	Let $X\sim\mathcal{N}(0,\Sigma)$ in $\RR^d$ with positive-definite covariance
	matrix~$\Sigma$. Let $v\in\RR^d$ with $\norm{v}=1$ and define the teacher
	partition
	\[
	\Omega_1^\star := \{x:\ip{v}{x}\ge0\},
	\qquad
	\Omega_2^\star := \{x:\ip{v}{x}<0\}.
	\]
	
	The teacher consists of two linear experts with shared baseline $m:\cX\to\RR$, where here $\cX=\RR^d$.
	\[
	f_1^\star(x) := m(x) + \tfrac12 d_*^\top x,
	\qquad
	f_2^\star(x) := m(x) - \tfrac12 d_*^\top x,
	\]
	for some fixed contrast vector $d_*\in\RR^d$. On each side of the separator the
	expert means differ by $d_*^\top x$, so the teacher contrast is linear with
	direction~$d_*$. For clarity, we consider a noiseless teacher.
	\[
	Y = f_1^\star(X)\1_{X\in\Omega_1^\star}
	+ f_2^\star(X)\1_{X\in\Omega_2^\star}.
	\]
	The case with additive bounded noise can be handled with minor modifications and
	does not change the symmetry-breaking conclusions (it adds only bounded
	perturbations to the moment expressions below and does not change their signs
	under the same non-degeneracy assumptions).
	
	The student is a two-expert MoE with linear experts
	\[
	f_1(x;\theta_1) = w_1^\top x + b_1,
	\qquad
	f_2(x;\theta_2) = w_2^\top x + b_2,
	\]
	and a router that uses antisymmetric scores
	\[
	a_1(x;u) = u^\top x,
	\qquad
	a_2(x;u) = -u^\top x,
	\]
	for a router parameter $u\in\RR^d$. The corresponding softmax weights are
	\[
	p_1^{(\tau)}(x;u)
	= \sigma\Bigl(\frac{2u^\top x}{\tau}\Bigr),
	\qquad
	p_2^{(\tau)}(x;u) = 1-p_1^{(\tau)}(x;u),
	\]
	where $\sigma(z)=1/(1+e^{-z})$ and $\tau>0$ is the router temperature. The
	student prediction is
	\[
	h_{\theta,u,\tau}(x)
	:= p_1^{(\tau)}(x;u) f_1(x;\theta_1)
	+ p_2^{(\tau)}(x;u) f_2(x;\theta_2).
	\]
	For fixed temperature in this subsection, we use the shorthand $h(X;u,\theta):=h_{\theta,u,\tau}(X)$.
	
	We consider gradient descent on the squared risk
	\[
	L_\tau(u,\theta)
	:= \EE\bigl[(Y-h(X;u,\theta))^2\bigr],
	\]
	starting from a \emph{symmetric initialization} where
	\[
	u\approx 0,
	\qquad
	(w_1,b_1)\approx(w_2,b_2),
	\]
	so that initially the two experts are nearly identical and the router is
	nearly balanced. At exact symmetry, $f_1=f_2$ and $h(\cdot;u,\theta)$ is
	independent of~$u$, reflecting the permutation symmetry between experts. In
	practice, small random perturbations in the expert parameters or stochastic
	gradients break this perfect symmetry and induce a small but non-zero expert
	contrast. The question is then whether the router amplifies this contrast and
	aligns with the teacher separator $v$, or whether the symmetric state remains
	stable.
	
	\subsection{Linearization and effective operator}\label{subsec:linearization}
	
	We do not try to derive the full coupled gradient-flow system for router and experts. That system contains several effects that are irrelevant to the local question here. Instead, the reduced model looks at the onset of symmetry breaking after a small expert contrast has appeared and before the router has moved far from balance. We linearize the router dynamics around $u=0$, treat the expert contrast as a slowly varying input, and identify the effective operator that decides whether the symmetric router is locally stable.
	
	For fixed expert parameters $\theta:=(\theta_1,\theta_2)$, the gradient of the
	risk with respect to $u$ is
	\begin{equation}\label{eq:grad-u-exact}
		\nabla_u L_\tau(u,\theta)
		= -2\,\EE\Bigl[(Y-h(X;u,\theta))\,\partial_u h(X;u,\theta)\Bigr].
	\end{equation}
	Since
	\[
	h_{\theta,u,\tau}(x)
	= p_1^{(\tau)}(x;u)\,f_1(x;\theta_1)
	+ p_2^{(\tau)}(x;u)\,f_2(x;\theta_2),
	\]
	and
	\[
	\partial_u p_1^{(\tau)}(x;u)
	= \frac{2}{\tau}\,p_1^{(\tau)}(x;u)p_2^{(\tau)}(x;u)\,x,
	\qquad
	\partial_u p_2^{(\tau)}(x;u) = -\partial_u p_1^{(\tau)}(x;u),
	\]
	we obtain
	\[
	\partial_u h_{\theta,u,\tau}(x)
	= \frac{2}{\tau}\,p_1^{(\tau)}(x;u)p_2^{(\tau)}(x;u)\,\bigl(f_1(x;\theta_1)-f_2(x;\theta_2)\bigr)\,x.
	\]
	
	Substituting into~\eqref{eq:grad-u-exact} gives
	\begin{equation}\label{eq:grad-u-psi}
		\nabla_u L_\tau(u,\theta)
		= -\frac{4}{\tau}\,
		\EE\Bigl[(Y-h(X;u,\theta))\,p_1^{(\tau)}(X;u)p_2^{(\tau)}(X;u)\,
		\bigl(f_1(X;\theta_1)-f_2(X;\theta_2)\bigr)\,X\Bigr].
	\end{equation}
	
	Near a symmetric configuration where $u$ is small and the experts have already started to align with the teacher, the expression simplifies in two complementary ways. First, for fixed $\tau$ and small $u$, the factor $p_1p_2$ is approximately $1/4$; more precisely, its Taylor expansion has the form $p_1^{(\tau)}p_2^{(\tau)}=1/4+O((u^\top X)^2/\tau^2)$ on bounded score regions. Thus the radius of this linearization shrinks with $\tau$, and the reduced calculation should be read locally in $u$ for the chosen temperature. Second, the signed product $(Y-h)\bigl(f_1-f_2\bigr)$ is precisely the quantity that records whether the current expert contrast is being \emph{reinforced} by the residual. In symmetry-breaking regimes this product correlates with the teacher side through $v^\top X$ and changes sign when that teacher side is flipped.
	
	A genuine linearization in $u$ must produce a term of the form $Mu$, not a $u$-independent drift. This requires the signed product $(Y-h)(f_1-f_2)$ to depend to first order on the router score $u^\top X$ in an odd way. Rather than deriving that dependence from the full coupled training dynamics, we isolate it as an explicit \emph{response assumption}. For $u$ near the origin, assume there exists a measurable function $\psi:\RR\times\RR\to\RR$, depending on the router score $s:=u^\top X$ and teacher score $t:=v^\top X$, such that
	\begin{equation}\label{eq:response-assumption}
		(Y-h(X;u,\theta))\bigl(f_1(X;\theta_1)-f_2(X;\theta_2)\bigr)
		= \psi(u^\top X, v^\top X)\,(d_*^\top X) + r_u(X).
	\end{equation}
	We assume that $\psi(\cdot,t)$ is odd in its first argument for each fixed $t$, that it is differentiable at $s=0$, and that $\partial_s\psi(0,t)=:g(t)$ for a measurable scalar function $g$. The derivative $g$ is assumed odd and satisfies $g(z)z\ge0$ for all $z$, so the feedback is aligned with the teacher side rather than opposed to it. Finally, the first-order expansion is assumed in the weighted $L^1$ sense needed by the gradient calculation:
	\begin{equation}\label{eq:response-weighted-L1}
	\EE\Bigl[
	\bigl|\psi(u^\top X,v^\top X)-(u^\top X)g(v^\top X)\bigr|\,|d_*^\top X|\,\|X\|
	\Bigr]=o(\|u\|),
	\end{equation}
	and the residual obeys $\EE[|r_u(X)|\,\|X\|]=o(\|u\|)$ as $u\to0$. This assumption should be read as a reduced linearized model for the early symmetry-breaking regime, not as a theorem about the full coupled MoE dynamics. In plain terms, it says the residual--contrast product reacts linearly to a small router imbalance, the sign of that reaction is tied to the teacher side $v^\top X$, and the resulting feedback reinforces rather than cancels the existing expert contrast.
	
	With the response assumption in place, the linearization becomes a bookkeeping calculation. For fixed temperature, the softmax factor satisfies $p_1^{(\tau)}p_2^{(\tau)}=1/4+O((u^\top X)^2/\tau^2)$ near $u=0$; under the Gaussian moments in this subsection, its contribution beyond $1/4$ is $o(\|u\|)$ in the gradient expansion for fixed $\tau$. The weighted $L^1$ expansion~\eqref{eq:response-weighted-L1} controls the response term in exactly the norm that appears after multiplication by $X$.
	Substituting these expansions into \eqref{eq:grad-u-psi} and using the remainder condition on $r_u$ gives
	\begin{equation}\label{eq:grad-u-linearized-correct}
		\nabla_u L_\tau(u,\theta)
		= -\frac{1}{\tau}\,\EE\bigl[(u^\top X)\,g(v^\top X)\,(d_*^\top X)\,X\bigr]
		+ o(\norm{u})
		= -\frac{1}{\tau}\,M u + o(\norm{u}),
	\end{equation}
	where the \emph{effective linear operator} $M$ is the symmetric matrix
	\begin{equation}\label{eq:M-definition}
		M := \EE\bigl[ g(v^\top X)\,(d_*^\top X)\,X X^\top \bigr].
	\end{equation}
	(The matrix is symmetric because it is an expectation of a scalar times
	$XX^\top$.)
	
	Under~\eqref{eq:grad-u-linearized-correct}, a gradient descent step with
	learning rate $\eta>0$ yields the linearized dynamics
	\begin{equation}\label{eq:linear-dynamics}
		u_{t+1}
		= u_t - \eta\nabla_u L_\tau(u_t,\theta)
		= \Bigl(I + \frac{\eta}{\tau}M\Bigr) u_t + o(\|u_t\|).
	\end{equation}
	
	The spectrum of~$M$ controls whether the symmetric state $u\approx0$ is stable
	or unstable. Directions corresponding to positive eigenvalues of~$M$ grow under
	gradient descent, while those with negative eigenvalues decay.
	
	\subsection{Gaussian moments along the teacher separator}
	
	It remains to see whether the effective operator has a direction aligned with the teacher separator. In the Gaussian case this can be checked explicitly. We first give a formula for any odd response function $g$ that depends on~$X$ only through $v^\top X$, and then specialize to concrete choices.
	
	\begin{lemma}[Gaussian moment along the teacher separator]\label{lem:Gaussian-M-general}
		Let $X\sim\mathcal{N}(0,\Sigma)$, $\norm{v}=1$, and suppose the response
		function has the form $g(v^\top X)$ where
		$g:\RR\to\RR$ is measurable, odd, and satisfies
		$g(z)z\ge0$ for all $z$ and
		$\EE[\abs{g(v^\top X)}\abs{v^\top X}^3]<\infty$. Let $M$ be given
		by~\eqref{eq:M-definition}. Then
		\[
		v^\top M v
		= \kappa_g\,(d_*^\top\Sigma v)\,\sqrt{v^\top\Sigma v},
		\]
		where for $s^2:=v^\top\Sigma v$ and $G\sim\mathcal{N}(0,s^2)$,
		\[
		\kappa_g
		:= \frac{\EE[g(G)G^3]}{s^{3}}.
		\]
		In particular, if $g$ is not identically zero and satisfies
		$g(z)\ge0$ for $z\ge0$, then $\kappa_g>0$ and
		\[
		\mathrm{sign}\bigl(v^\top M v\bigr)
		= \mathrm{sign}\bigl(d_*^\top\Sigma v\bigr).
		\]
	\end{lemma}
	
	\begin{proof}
		Let $Z:=v^\top X\sim\mathcal{N}(0,s^2)$ with $s^2:=v^\top\Sigma v$.
		By Gaussian regression,
		\[
		\EE[X\mid Z=z]
		= \frac{\Sigma v}{v^\top\Sigma v}\,z
		=: \alpha z,
		\qquad
		\Cov(X\mid Z)
		= \Sigma - \frac{\Sigma v v^\top \Sigma}{v^\top\Sigma v}
		=: \Sigma_\perp,
		\]
		and we can write $X=\alpha Z+\varepsilon$ with
		$\varepsilon\sim\mathcal{N}(0,\Sigma_\perp)$ independent of~$Z$.
		
		We compute
		\[
		v^\top M v
		= \EE\bigl[g(Z)\,(d_*^\top X)\,(v^\top X)^2\bigr]
		= \EE\bigl[g(Z)\,(d_*^\top X)\,Z^2\bigr].
		\]
		Conditioning on $Z$ and using $\EE[\varepsilon\mid Z]=0$ gives
		\[
		\EE[d_*^\top X\mid Z=z]
		= d_*^\top(\alpha z)
		= (d_*^\top\Sigma v)\,\frac{z}{s^2}.
		\]
		Hence
		\[
		v^\top M v
		= (d_*^\top\Sigma v)\,\frac{1}{s^2}\,\EE[g(Z)Z^3].
		\]
		Writing $Z$ as $G\sim\mathcal{N}(0,s^2)$ yields
		\[
		v^\top M v
		= (d_*^\top\Sigma v)\,\frac{1}{s^2}\,\EE[g(G)G^3]
		= \kappa_g\,(d_*^\top\Sigma v)\,s,
		\]
		with $\kappa_g:=\EE[g(G)G^3]/s^3$ and $s=\sqrt{v^\top\Sigma v}$.
		
		If $g$ is odd and $g(z)\ge0$ for $z\ge0$, then
		\[
		\EE[g(G)G^3]
		= 2\int_0^\infty g(z)z^3
		\frac{1}{\sqrt{2\pi}s}e^{-z^2/(2s^2)}\,\mathrm{d}z > 0,
		\]
		because the integrand is non-negative for all $z\ge0$ and strictly positive
		on a set of positive measure (since $g$ is not identically zero).
	\end{proof}
	
	Two concrete choices of $g$ are particularly natural. One is the sign response $g(z)=\mathrm{sign}(z)$, which corresponds to a router that reacts only to the side of the teacher hyperplane. The other is the smooth response $g(z)=\tanh(\gamma z)$, with $\gamma>0$, which is more faithful to the softmax router. For the sign response we recover an explicit formula.
	
	\begin{lemma}[Alignment along the teacher separator for sign response]\label{lem:Gaussian-M}
		Let $X\sim\mathcal{N}(0,\Sigma)$, $g(z)=\mathrm{sign}(z)$ with
		$\norm{v}=1$, and $M$ as in~\eqref{eq:M-definition}. Then
		\[
		v^\top M v
		= c_G\,(d_*^\top\Sigma v)\,\sqrt{v^\top\Sigma v},
		\qquad
		c_G=2\sqrt{\frac{2}{\pi}}>0.
		\]
		In particular, if $d_*^\top\Sigma v\neq0$ then $v^\top M v\neq0$.
	\end{lemma}
	
	\begin{proof}
		This is the special case of Lemma~\ref{lem:Gaussian-M-general} with
		$g(z)=\mathrm{sign}(z)$, so $g(G)G^3=\abs{G}^3$ for $G\sim\mathcal{N}(0,s^2)$.
		A standard Gaussian moment computation gives
		\[
		\EE[\abs{G}^3]
		= 2\sqrt{\frac{2}{\pi}}\,s^3.
		\]
		Thus $\kappa_g=2\sqrt{2/\pi}$ and the stated formula follows.
	\end{proof}
	
	For the smooth response $g(z)=\tanh(\gamma z)$ one obtains a similar formula
	with a different positive constant
	$\kappa_g=\kappa(\gamma)>0$; the sign and qualitative behavior are the same.
	
	\subsection{Symmetry breaking and alignment}\label{subsec:symmetry-breaking-main}
	
	The linearization and the Gaussian moment calculation now give the symmetry-breaking statement. Whenever the teacher contrast $d_*$ has a nonzero projection along the separator $v$, the balanced router is linearly unstable in at least one direction correlated with~$v$.
	
	A notational simplification is convenient before the theorem. Since the teacher partition is unchanged if we replace $v$ by $-v$ and swap $\Omega_1^\star,\Omega_2^\star$, we may assume without loss of generality that
	\begin{equation}\label{eq:nondeg-wlog}
		d_*^\top\Sigma v>0,
	\end{equation}
	whenever $d_*^\top\Sigma v\neq 0$.
	
	The proposition is deliberately local. It does not describe the whole training trajectory. It identifies the first instability seen by the symmetric router and shows that this unstable direction must have a nonzero component along the teacher separator.
	
	\begin{proposition}[Reduced linear instability of the symmetric router]\label{prop:symmetry-breaking}
		In the Gaussian two-expert setting above, assume the response expansion
		\eqref{eq:response-assumption} holds with $\partial_s\psi(0,t)=g(t)$ where
		$g$ is odd, $g(z)z\ge0$, and not identically zero. Assume also
		$d_*^\top\Sigma v\neq0$ (equivalently, after possibly flipping $v$,
		\eqref{eq:nondeg-wlog} holds). Let $M$ be as in~\eqref{eq:M-definition}. Then the Rayleigh quotient has the following explicit form.
		\[
		v^\top M v
		= \kappa_g\,(d_*^\top\Sigma v)\,\sqrt{v^\top\Sigma v},
		\qquad \kappa_g>0.
		\]
		In particular $v^\top M v>0$ under~\eqref{eq:nondeg-wlog}. Consequently, $M$ has at least one positive eigenvalue. Moreover, there is an eigenvector $w_+$ with eigenvalue $\lambda_+>0$ such that
		\[
		\ip{w_+}{v}\neq0,
		\]
		so the unstable eigendirection has nontrivial correlation with the teacher separator. For gradient descent with step size $\eta>0$, the linearized update $u_{t+1}\approx(I+\eta M/\tau)u_t$ is unstable along $w_+$ whenever $1+\eta\lambda_+/\tau>1$, equivalently whenever $\eta\lambda_+/\tau>0$. Thus, for fixed $\eta>0$ and sufficiently small~$\tau$, the symmetric point $u=0$ is linearly unstable and the router parameter $u_t$ is driven toward a direction correlated with~$v$.
	\end{proposition}
	
	\begin{proof}
		The Rayleigh-quotient identity is exactly Lemma~\ref{lem:Gaussian-M-general}. Under the standing assumptions on $g$, $\kappa_g>0$, so the sign of $v^\top M v$ is the sign of $d_*^\top\Sigma v$.
		
		To obtain the spectral consequence, let $\{(w_i,\lambda_i)\}_{i=1}^d$ be an orthonormal
		eigen-decomposition. Then
		\[
		v^\top M v
		= \sum_{i=1}^d \lambda_i\,\ip{v}{w_i}^2.
		\]
		If $v^\top M v>0$, at least one summand must be strictly positive, so there
		exists an index $i$ with $\lambda_i>0$ and $\ip{v}{w_i}\neq0$. Taking
		$w_+:=w_i$ and $\lambda_+:=\lambda_i$ proves the claim.
		
		The instability statement then follows by writing the linearized update as
		\[
		u_{t+1}\approx A u_t,
		\qquad
		A:=I+\frac{\eta}{\tau}M.
		\]
		If $u_0$ has a nonzero component along $w_+$, then that component is multiplied
		(at the linearized level) by $1+\eta\lambda_+/\tau$ at each step. Thus if
		$\eta\lambda_+/\tau>0$ (equivalently $1+\eta\lambda_+/\tau>1$), that component
		grows exponentially, giving linear instability. The temperature dependence is explicit in the multiplicative factor $1+\eta\lambda_+/\tau$. For fixed $\eta,\lambda_+>0$, the map $\tau\mapsto 1+\eta\lambda_+/\tau$ is strictly decreasing in $\tau$, hence smaller $\tau$ produces a larger amplification per step.
	\end{proof}
	
	Proposition~\ref{prop:symmetry-breaking} formalizes the intuitive picture: once the
	experts have developed a non-trivial contrast $d_*$ with a component along the
	teacher separator~$v$, the symmetric router $u=0$ becomes linearly unstable in
	an eigendirection that is necessarily correlated with $v$. The exact per-step
	amplification factor is $1+\eta\lambda_+/\tau$ in the linearized iteration, so smaller router temperature strengthens the local instability for fixed step size.
	
	\subsection{Alignment rate and temperature dependence}\label{subsec:alignment-rate}
	
	The same linearization also gives the alignment time scale, but this is only a direct spectral consequence of the update in Proposition~\ref{prop:symmetry-breaking}. Let $\lambda_1>\lambda_2\ge\cdots\ge\lambda_d$ be the eigenvalues of~$M$, with orthonormal eigenbasis $\{w_i\}_{i=1}^d$, and suppose $\lambda_1>0$. If $u_0=\sum_i\alpha_i^{(0)}w_i$ with $\alpha_1^{(0)}\ne0$, then the linearized update $u_{t+1}=(I+\eta M/\tau)u_t$ gives
	\begin{equation}\label{eq:linearized-spectral-components}
		\alpha_i^{(t)}
		= \Bigl(1+\frac{\eta}{\tau}\lambda_i\Bigr)^t\alpha_i^{(0)}.
	\end{equation}
	Under the sign-stability condition $1+\eta\lambda_d/\tau>0$, which is automatic for sufficiently small step size at fixed temperature, the non-leading components decay relative to the leading one according to
	\begin{equation}\label{eq:alignment-ratio}
		\left|\frac{\alpha_i^{(t)}}{\alpha_1^{(t)}}\right|
		= \left|\frac{1+\eta\lambda_i/\tau}{1+\eta\lambda_1/\tau}\right|^t
		\left|\frac{\alpha_i^{(0)}}{\alpha_1^{(0)}}\right|
		\longrightarrow 0,
		\qquad i\ge2.
	\end{equation}
	Thus the normalized direction $u_t/\|u_t\|$ converges to the leading eigendirection up to sign whenever the leading coefficient is nonzero and the leading eigenvalue is simple. Writing
	\begin{equation}\label{eq:alignment-rho}
		\rho(\tau):=\left|\frac{1+\eta\lambda_2/\tau}{1+\eta\lambda_1/\tau}\right|,
	\end{equation}
	the non-leading-to-leading ratio decays geometrically like $\rho(\tau)^t$. In the small-step regime $\eta/\tau\ll1$,
	\begin{equation}\label{eq:alignment-small-step}
		\rho(\tau)=1-\frac{\eta}{\tau}(\lambda_1-\lambda_2)+O\!\left((\eta/\tau)^2\right),
		\qquad
		t_{\mathrm{align}}\asymp \frac{\tau}{\eta(\lambda_1-\lambda_2)}.
	\end{equation}
	For fixed $\eta$ and $\tau\downarrow0$ with $\lambda_1,\lambda_2>0$, this positive-eigenspace ratio satisfies $\rho(\tau)\to\lambda_2/\lambda_1$, so the alignment rate saturates rather than continuing to scale like $1/\tau$. If some eigenvalues are negative and $\eta$ is held fixed, the sign-stability condition can fail as $\tau\downarrow0$; the displayed alignment-rate calculation should then be read either at fixed temperature with a stable step size or on the positive spectral subspace.
	
	The scope of the symmetry-breaking statement is local. It concerns a short-time, near-symmetric linearized update in which the expert contrast is treated as a slowly varying input to the router. In that regime, smaller $\tau$ increases the multiplicative factor $1+\eta\lambda_+/\tau$ and accelerates the onset of specialization. The result does not imply global convergence of the full coupled MoE training dynamics; it provides one tractable local route by which a symmetric router can move toward a teacher-aligned partition.

	\section{Diagnostic experiments}\label{sec:experiments}
	
	We now turn from the estimates above to a small set of diagnostics. The purpose of this section is deliberately narrow. The experiments are not meant to compare large sparse-MoE training recipes, nor to claim that the reduced models capture all details of modern MoE systems. Instead, they ask whether the quantities that appear in the proofs are visible in controlled numerical examples where the teacher router, the interface, the temperature, and the input distribution are all known. The plan is to isolate the same boundary-layer mechanism in three ways. We first study the temperature scaling of the soft--hard gap, then move a fixed-form boundary through regions of different input density, and finally revisit the local symmetry-breaking calculation for a nearly balanced two-expert router.
	
	This is why the experiments use synthetic teacher models. The router, the decision boundary, the expert contrast, and the hard-routing predictor are known, so the boundary mass and the soft--hard discrepancy can be measured directly. A real MoE training run would mix this mechanism with finite-sample noise, misspecification, load balancing, optimizer choices, expert-capacity effects, and architecture-specific details. Those effects are important for practice, but they are not the object of the theory developed above. In the controlled setting used here, we can change one quantity at a time and ask whether the measured behavior follows the boundary-layer picture.
	
	All experiments use fixed random seeds and Monte Carlo estimates of the corresponding population quantities. The responses are noiseless, so the reported gaps are not test errors in the usual benchmark sense. They measure what changes when the hard router is replaced by its temperature-$\tau$ soft version, or when the hard assignment is perturbed near a known interface. The scripts distributed with the source write the CSV files, JSON summaries, plots, and \LaTeX{} table fragments used below; the numerical entries in the manuscript are read from those files rather than transcribed by hand.
	
	\subsection{Boundary-layer scaling}
	
	The first experiment is the most direct check of the soft-to-hard estimate. We take a two-expert teacher with $X\sim \mathcal{N}(0,I_d)$, a linear hard router, and two linear experts. The hard teacher predicts
	\[
		f_0(x)=\1\{w_{\mathrm r}^\top x-b\ge 0\}\,\beta_1^\top x+
		\1\{w_{\mathrm r}^\top x-b<0\}\,\beta_2^\top x,
	\]
	and the soft student uses the same experts with the logistic gate $\sigma((w_{\mathrm r}^\top x-b)/\tau)$. Thus the two predictors differ only in a neighborhood of the separating hyperplane. In this realizable setup the hard risk is zero, so $\widehat L_\tau-\widehat L_0$ is a direct Monte Carlo estimate of the soft--hard population gap. We compare that gap with the empirical boundary mass $\widehat\PP\{|w_{\mathrm r}^\top X-b|\le 2\tau\}$.
	
	\begin{table}[H]
\centering
\caption{Boundary-layer scaling in the two-expert teacher model. The last column is stable at small temperatures, consistent with the $O(\tau)$ boundary-layer prediction.}
\label{tab:exp-boundary-scaling}
\begin{tabular}{cccc}
\toprule
$\tau$ & $\widehat P\{|\Delta|\leq 2\tau\}$ & $\widehat L_\tau-\widehat L_0$ & $(\widehat L_\tau-\widehat L_0)/\tau$ \\
\midrule
0.020 & 0.0317 & 0.01161 & 0.5805 \\
0.050 & 0.0797 & 0.02969 & 0.5938 \\
0.100 & 0.1587 & 0.05994 & 0.5994 \\
0.200 & 0.3116 & 0.11989 & 0.5994 \\
0.450 & 0.6327 & 0.26180 & 0.5818 \\
0.600 & 0.7700 & 0.33880 & 0.5647 \\
\bottomrule
\end{tabular}
\end{table}

	\begin{figure}[H]
		\centering
		\includegraphics[width=.48\textwidth]{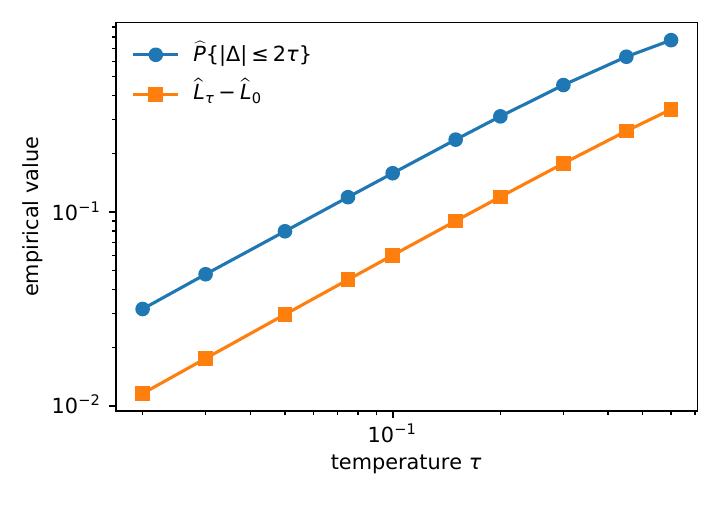}\hfill
		\includegraphics[width=.48\textwidth]{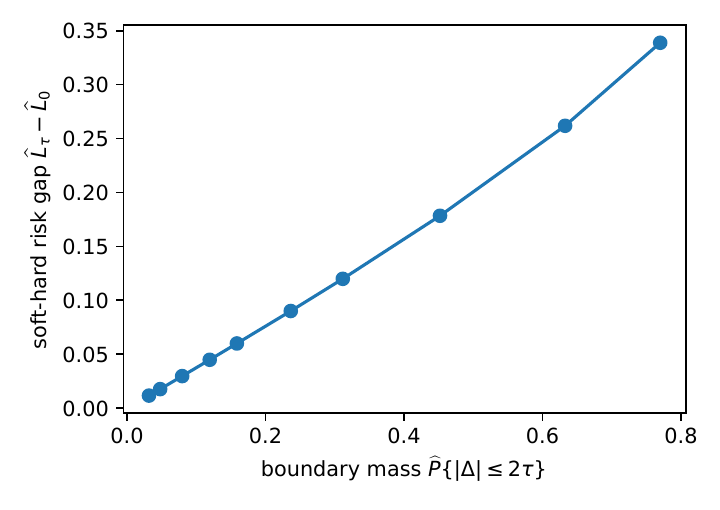}
		\caption{Boundary-layer scaling. The left panel shows that both the empirical boundary mass and the soft--hard risk gap scale nearly linearly at small temperatures. The right panel shows that the risk gap is nearly linear in the measured boundary mass. In this run, the small-temperature log--log slopes are $0.994$ for boundary mass and $1.014$ for the risk gap, and the correlation between boundary mass and risk gap over the displayed temperature grid is $0.998$.}
		\label{fig:exp-boundary-scaling}
	\end{figure}
	
	Table~\ref{tab:exp-boundary-scaling} and Figure~\ref{fig:exp-boundary-scaling} should be read together. The table gives the numerical scale of the effect, while the figure shows the two comparisons that matter for the theory. In the left panel, decreasing $\tau$ shrinks both the measured slab probability and the soft--hard risk gap, and the two log--log slopes are close to one. This is the visual version of the coarea prediction. For a regular interface, the relevant probability mass in a thin tube is first order in the tube width. The right panel removes $\tau$ from the horizontal axis and plots the gap directly against the measured boundary mass. The near-linear relation is the main diagnostic message. In this example, the risk gap is not behaving like an unrelated temperature artifact; it is tracking the amount of probability placed near the routing interface.

	The point of the experiment is not that this Gaussian example is difficult. It is deliberately simple so that the interface is known and the measured quantity in the code can be compared with the boundary-layer quantity that appears in the proof. The nearly constant values of $(\widehat L_\tau-\widehat L_0)/\tau$ at the smaller temperatures give the same message in tabular form.
	
	\subsection{Controlled boundary geometry}
	
	The second experiment keeps the temperature fixed and changes the location of the interface. This separates a temperature effect from a geometry effect. We translate the boundary $w_{\mathrm r}^\top x=b$ through the Gaussian input distribution and choose the expert contrast to be tangent to the translated interfaces. With this choice, moving the boundary mainly changes how much Gaussian probability lies near the interface; it does not substantially change the leading squared-contrast factor in the coarea coefficient.
	
	For each offset we measure two quantities. The first is again the soft--hard risk gap. The second is the hard-assignment flip rate under a fixed small router perturbation. This perturbation is not a trained update and is not meant to model a full optimizer step. It is a stress test. When many points lie close to the decision boundary, a small change in the router should move many of them across the boundary. When little mass lies near the boundary, the same perturbation should have a much smaller effect.
	
	\begin{table}[H]
\centering
\caption{Geometry-controlled experiment at fixed temperature $\tau=0.10$. The expert contrast is chosen tangent to the translated interfaces, so moving the boundary primarily changes the Gaussian density near the interface. Both the soft-hard risk gap and the hard-assignment flip rate decrease with the measured boundary mass.}
\label{tab:exp-geometry}
\begin{tabular}{cccc}
\toprule
offset $b$ & boundary mass & risk gap & flip rate \\
\midrule
0.00 & 0.1590 & 0.06063 & 0.0252 \\
0.50 & 0.1405 & 0.05407 & 0.0226 \\
1.00 & 0.0966 & 0.03691 & 0.0154 \\
1.50 & 0.0523 & 0.01996 & 0.0084 \\
2.00 & 0.0220 & 0.00846 & 0.0035 \\
2.25 & 0.0130 & 0.00489 & 0.0020 \\
\bottomrule
\end{tabular}
\end{table}

	\begin{figure}[H]
		\centering
		\includegraphics[width=.48\textwidth]{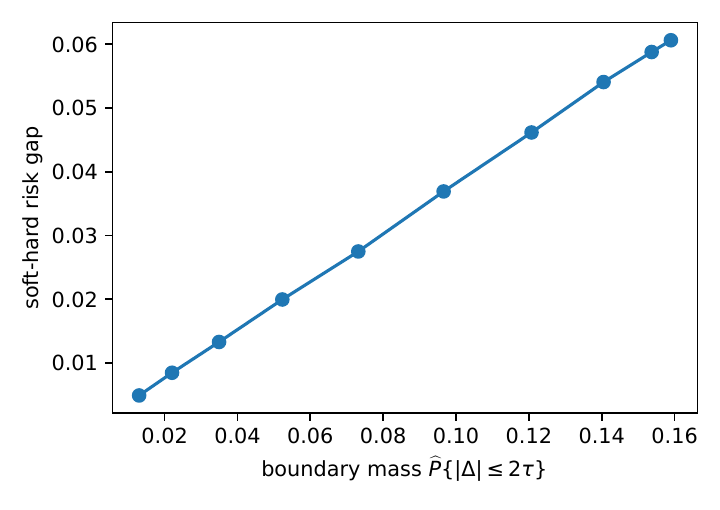}\hfill
		\includegraphics[width=.48\textwidth]{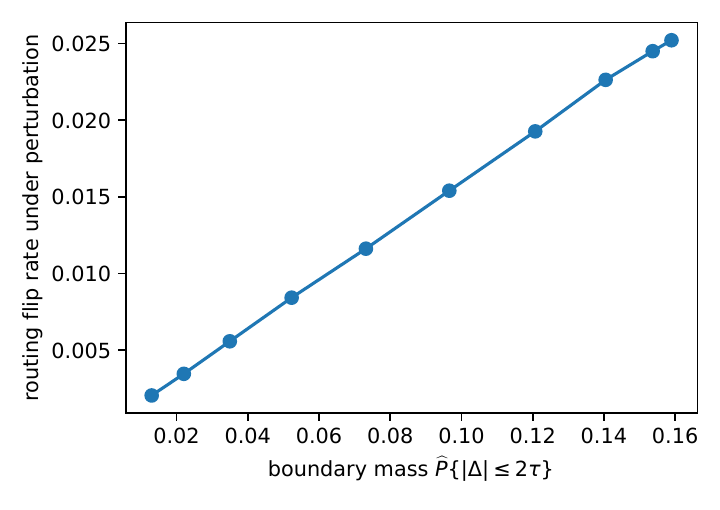}
		\caption{Controlled boundary geometry at fixed temperature $\tau=0.10$. Moving the interface into lower-density regions reduces boundary mass. Because the expert contrast is tangent to the translated interfaces, the risk gap and the hard-assignment flip rate closely track this reduction; both empirical correlations with boundary mass exceed $0.9999$.}
		\label{fig:exp-boundary-geometry}
	\end{figure}
	
	Table~\ref{tab:exp-geometry} and Figure~\ref{fig:exp-boundary-geometry} show the expected monotone pattern. At the same temperature, the soft--hard discrepancy is larger when the interface passes through a high-density region and smaller when the interface is moved into the tail of the input distribution. The left panel of the figure makes this more explicit by plotting the risk gap against the measured boundary mass rather than against the offset $b$. Once the boundary mass is known, the translated interfaces fall almost on a single line in this controlled design.

	The right panel gives a second interpretation of the same geometry. A small router perturbation flips more hard assignments when there is more probability close to the interface, and it flips fewer assignments when the interface lies in a lower-density region. This is not a separate theorem, but it is a useful sanity check. A boundary layer with more mass should be both more important for soft--hard approximation and more sensitive to small boundary motion. Thus the figure separates the two roles played by temperature and geometry. The temperature fixes the thickness of the layer; the router geometry and input density determine how much probability lies inside it.
	
	\subsection{Reduced symmetry-breaking diagnostic}
	
	The last experiment concerns the most model-dependent part of the paper, the reduced symmetry-breaking calculation in Section~\ref{sec:symmetry-breaking}. We generate a two-expert teacher with a linear separator and then fix two student experts with a small contrast around the best shared linear predictor. Only the router is optimized. The router is initialized close to balanced, with initial alignment $0.05$ against the teacher separator, and the same Monte Carlo sample and initial direction are used for all displayed temperatures.
	
	This setup intentionally removes many features of full MoE training. The experts are not updated, there is no load-balancing loss, and the experiment is run on a fixed Monte Carlo approximation to the population objective. These restrictions are the point. The calculation in Section~\ref{sec:symmetry-breaking} assumes that a small expert contrast is already present and studies whether the local router update amplifies the teacher-aligned direction. The experiment checks that narrow local issue, not global convergence of coupled router--expert training.
	
	\begin{table}[H]
\centering
\caption{Reduced symmetry-breaking diagnostic. The experts are fixed with a small contrast and only the router is optimized. Starting from alignment $0.05$, the router aligns with the teacher separator in this controlled Monte Carlo approximation to the population objective.}
\label{tab:exp-training}
\begin{tabular}{cccccc}
\toprule
$\tau$ & final risk & final alignment & $\|\hat u\|$ & boundary mass & gate entropy \\
\midrule
0.05 & 1.2601 & 1.000 & 1.430 & 0.056 & 0.046 \\
0.10 & 1.2748 & 1.000 & 1.997 & 0.080 & 0.065 \\
0.20 & 1.2960 & 1.000 & 2.780 & 0.114 & 0.093 \\
0.40 & 1.3268 & 1.000 & 3.840 & 0.163 & 0.133 \\
\bottomrule
\end{tabular}
\end{table}

	\begin{figure}[H]
		\centering
		\includegraphics[width=.48\textwidth]{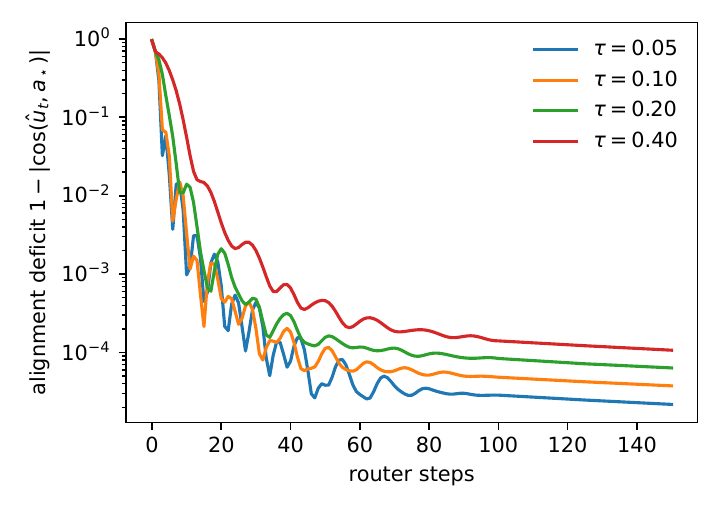}\hfill
		\includegraphics[width=.48\textwidth]{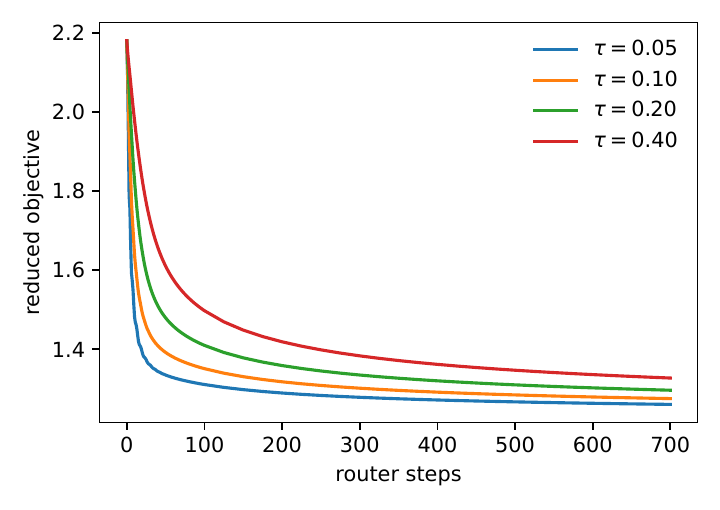}
		\caption{Reduced symmetry-breaking diagnostic. The left panel plots the early-time alignment deficit $1-|\cos(\hat u_t,a_\star)|$ on a logarithmic scale, rather than the raw alignment, because the raw alignment saturates quickly and otherwise hides the transient. The right panel plots the corresponding reduced objective. The experiment is a controlled mechanism check for Proposition~\ref{prop:symmetry-breaking}, not a claim about global convergence of full MoE training.}
		\label{fig:exp-symmetry-breaking}
	\end{figure}
	
	Table~\ref{tab:exp-training} and Figure~\ref{fig:exp-symmetry-breaking} summarize the reduced training run from two angles. The table records the endpoint quantities. The final alignments are close to one for all displayed temperatures, while the final gate entropy and measured boundary mass increase with temperature, as expected for a softer gate. The figure focuses on the part of the trajectory that is easiest to miss if one plots only raw alignment. In the left panel, the vertical axis is the remaining alignment deficit. The log scale makes the first few dozen router steps visible and shows that the router direction rapidly locks onto the teacher separator. In the right panel, the corresponding objective decreases along the same runs. The curves should not be compared as a benchmark across temperatures, because each temperature defines a slightly different soft objective; the relevant point is that each run follows the teacher-aligned direction once the small expert contrast is present.

	This is consistent with the reduced theory in Section~\ref{sec:symmetry-breaking}. Temperature does not mainly change the final direction in this reduced experiment. Instead, it changes the softness of the gate and can be partly offset by the learned router norm through the ratio $u^\top x/\tau$. That is why raw alignment alone is a poor diagnostic after the first transient. All runs select essentially the same direction. The endpoint quantities in Table~\ref{tab:exp-training} therefore matter. They show that higher temperature leaves a wider soft boundary layer and larger gate entropy, while the direction itself remains teacher-aligned. The experiment should not be read as evidence that arbitrary MoE training finds this configuration. It only shows that the local mechanism isolated in the proposition is visible in a controlled population-style calculation.
	
	Taken together, the diagnostics are consistent with the boundary-mass explanation developed in the paper. They provide three checks on the same idea. The soft--hard gap scales with the measured boundary layer, interface placement matters even at fixed temperature, and the local symmetry-breaking calculation can be reproduced when its assumptions are built into the experiment. The evidence is intentionally modest. These experiments do not address finite-sample generalization, load balancing, expert capacity constraints, or large-scale sparse-MoE systems.

		\section{Scope and interpretation}\label{sec:discussion}
	
	We close by spelling out what the preceding results do, and do not, establish. This matters because the paper combines three ingredients with different logical status. These are a boundary-layer analysis of the soft-to-hard limit, a conditional transfer statement about optimization geometry, and a reduced calculation that illustrates one way a routing direction can be selected. These ingredients fit together, but they should not be read as equally general.
	
	The most robust part of the paper is the boundary-layer analysis. Once a router and an expert family are fixed, the difference between the soft objective $L_\tau$ and the hard objective $L_0$ is governed by the amount of probability mass near the hard-routing interfaces. In the notation of the paper, this is the region where the top-two logit gap
	\[
	\Delta(x;\phi)
	\]
	is small. At temperature $\tau$, only inputs with $\Delta(x;\phi)$ on the order of $\tau$ can substantially feel the difference between soft and hard routing. Away from this thin set, the softmax weights are already exponentially close to a hard assignment. Thus the relevant quantity is not simply that the temperature is small, but that the input distribution places little mass in the boundary layer selected by the router.
	
	This part of the analysis does not rely on a teacher--student model, realizability of the experts, or a favorable training landscape. It is a geometric statement about low-temperature routing. The coarea estimates make this statement quantitative, the uniform approximation theorem turns it into a risk comparison, and the $\Gamma$-convergence result records the corresponding variational limit. The message is that soft-to-hard approximation is controlled by boundary mass, not by temperature alone.
	
	The landscape result in Section~\ref{sec:benign} has a more conditional role. It starts from a profiled hard-routing objective that already has the desired structure. The expert geometry separates the relevant components, the router can realize the corresponding interfaces, and the non-teacher critical points have a direction of escape. Under additional first- and second-order soft-to-hard control near the critical points being compared, that structure persists for sufficiently small temperature. In this sense the theorem is a stability statement. It does not say that arbitrary MoE objectives have benign landscapes, and it does not remove the need to understand the hard-routing problem. It says that, when the hard problem has the right geometry and the boundary-layer approximation is uniform enough, the soft objective inherits the useful part of that geometry.
	
	The Gaussian calculation in Section~\ref{sec:symmetry-breaking} should be read in the same spirit. It is deliberately reduced. The calculation does not model all of coupled router--expert training, and it does not include load balancing, capacity constraints, finite-sample effects, or architectural details of large sparse MoE systems. Its purpose is narrower. It shows in a tractable population setting how a small expert contrast can create a teacher-aligned direction for the router, and how the resulting soft objective can move away from the fully symmetric configuration. This gives a concrete mechanism that is compatible with the hard-routing picture, but the main boundary-mass theorems do not depend on it.
	
	The diagnostic experiments in Section~\ref{sec:experiments} were designed with the same separation in mind. They do not try to serve as benchmarks for sparse MoE training. Instead, they check whether the quantities appearing in the theory behave as predicted in controlled teacher models where the interface, temperature, expert contrast, and input density can be inspected directly. The experiments support the interpretation that the soft--hard discrepancy is a boundary-layer effect, that changing the geometry of the interface changes the approximation error even at fixed temperature, and that the reduced symmetry-breaking calculation is visible when its assumptions are built into the setup. These checks are useful, but intentionally modest.
	
	The main unresolved issue is therefore not the zero-temperature passage itself. The paper gives a route for controlling that passage once the hard partition is fixed or sufficiently well behaved. The harder remaining problem is to characterize the limiting hard-routing objective for realistic router--expert families. In particular, one would like conditions under which the profiled hard objective has identifiable partitions, useful curvature away from teacher-equivalent solutions, and enough stability under perturbations of the router and experts. Those questions are not settled here. The contribution of the paper is to isolate the boundary-mass part of the problem and to show, under explicit assumptions, why it is the quantity that governs the passage from soft routing to hard routing.

\end{document}